\documentclass{article}

\usepackage{microtype}
\usepackage{graphicx}
\usepackage{subfigure}
\usepackage{booktabs} % for professional tables

\usepackage{hyperref}

\usepackage[accepted]{icml2024}

\usepackage{amsmath}
\usepackage{amssymb}
\usepackage{mathtools}
\usepackage{amsthm}

\usepackage[capitalize,noabbrev]{cleveref}

\theoremstyle{plain}
\newtheorem{theorem}{Theorem}[section]
\newtheorem{proposition}[theorem]{Proposition}
\newtheorem{lemma}[theorem]{Lemma}

\theoremstyle{definition}
\newtheorem{definition}[theorem]{Definition}

\theoremstyle{remark}
\newtheorem{remark}[theorem]{Remark}
\newtheorem{example}[theorem]{Example}

\usepackage[textsize=tiny]{todonotes}

\usepackage{bm}
\usepackage{xcolor}
\usepackage{multicol}
\usepackage{multirow}

\def\1{\bm{1}}

\DeclareMathAlphabet{\mathsfit}{\encodingdefault}{\sfdefault}{m}{sl}
\SetMathAlphabet{\mathsfit}{bold}{\encodingdefault}{\sfdefault}{bx}{n}

\def\gA{{\mathcal{A}}}
\def\gB{{\mathcal{B}}}
\def\gC{{\mathcal{C}}}

\def\gF{{\mathcal{F}}}
\def\gG{{\mathcal{G}}}
\def\gH{{\mathcal{H}}}

\def\gL{{\mathcal{L}}}
\def\gM{{\mathcal{M}}}

\def\gO{{\mathcal{O}}}

\def\gT{{\mathcal{T}}}

\def\sD{{\mathbb{D}}}
\def\sL{{\mathbb{L}}}

\newcommand{\R}{\mathbb{R}}

\DeclareMathOperator*{\argmin}{arg\,min}

\newcommand{\inprod}[2]{\left\langle {#1}, {#2}\right\rangle}
\DeclareMathOperator{\diag}{diag}
\DeclareMathOperator{\cayley}{Cayley}
\newcommand{\prox}{\mathbf{prox}}

\newcommand{\dom}{\mathop{\bf dom}} % domain

\icmltitlerunning{Monotone, Bi-Lipschitz, and Polyak-\L{}ojasiewicz Networks\hfill\thepage}

\begin{document}

\twocolumn[
\icmltitle{Monotone, Bi-Lipschitz, and Polyak-\L{}ojasiewicz Networks}



\icmlsetsymbol{equal}{*}

\begin{icmlauthorlist}
\icmlauthor{Ruigang Wang}{acfr}
\icmlauthor{Krishnamurthy (Dj) Dvijotham}{google}
\icmlauthor{Ian R. Manchester}{acfr}
\end{icmlauthorlist}

\icmlaffiliation{google}{Google DeepMind}

\icmlaffiliation{acfr}{Australian Centre for Robotics, School of Aerospace, Mechanical and Mechatronic Engineering, The University of Sydney, Sydney, NSW 2006, Australia.}

\icmlcorrespondingauthor{Ruigang Wang}{ruigang.wang@sydney.edu.au}

\icmlkeywords{Machine Learning, ICML}

\vskip 0.3in
]



\printAffiliationsAndNotice{}  % leave blank if no need to mention equal contribution

\begin{abstract}
This paper presents a new \emph{bi-Lipschitz} invertible neural network, the BiLipNet, which has the ability to smoothly control both its \emph{Lipschitzness} (output sensitivity to input perturbations) and \emph{inverse Lipschitzness} (input distinguishability from different outputs). The second main contribution is a new scalar-output network, the PLNet, which is a composition of a BiLipNet and a quadratic potential. We show that PLNet satisfies the Polyak-\L{}ojasiewicz condition and can be applied to learn non-convex surrogate losses with a unique and efficiently-computable global minimum. The central technical element in these networks is a novel invertible residual layer with certified strong monotonicity and Lipschitzness, which we compose with orthogonal layers to build the BiLipNet. The certification of these properties is based on incremental quadratic constraints, resulting in much tighter bounds than can be achieved with spectral normalization. Moreover, we formulate the calculation of the inverse of a BiLipNet -- and hence the minimum of a PLNet -- as a series of three-operator splitting problems, for which fast algorithms can be applied.
\end{abstract}

\section{Introduction}
In many applications, it is desirable to learn neural networks with \emph{certified input-output behaviors}, i.e., certain properties that are guaranteed by design. For example, Lipschitz-bounded networks have proven to be beneficial for stabilizing of generative adversarial network (GAN) training \cite{arjovsky2017wasserstein,gulrajani2017improved}, certifying robustness against adversarial attacks \cite{tsuzuku2018lipschitz,singla2021skew,zhang2021towards,araujo2023unified,wang2023direct} and robust reinforcement learning \cite{russo2021towards,barbara2024robust}. 

\begin{figure}[!tb]
    \begin{minipage}{\linewidth}
    \centering
    \includegraphics[width=0.84\linewidth]{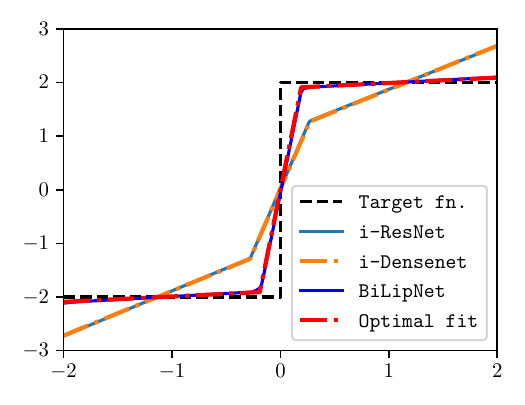}
    \end{minipage}
    \begin{minipage}{\linewidth}
    \centering
    \resizebox{0.82\textwidth}{!}{
    \begin{tabular}{c|ccc}
    \hline
    Model & inv. Lip. ($\downarrow$) & Lip. ($\uparrow$) & loss ($\downarrow$) \\
    \hline
    i-ResNet & 0.80 & 4.69 & 0.2090 \\
    i-DenseNet & 0.82 & 4.66 & 0.2091 \\
    \textbf{BiLipNet} & \textbf{0.11} & \textbf{9.97} & \textbf{0.0685} \\
    \hline
    Best Possible & 0.10 & 10.0 & 0.0677 \\
    \hline
    \end{tabular}
    }
    \end{minipage}
    \caption{Fitting a step function, which is not Lipschitz, with certified $(0.1, 10)$-Lipschitz models. Compared to the analytically-computed optimum, the proposed BiLipNet achieves much tighter bounds than models based on spectral normalization.}
    \label{fig:step}
\end{figure}

Another input-output property -- \emph{invertibility} has received much attention in the deep learning literature since the introduction of \emph{normalizing flows} \cite{dinh2014nice} for probability-density learning. Invertible neural networks have been applied in applications such as generative modeling \cite{dinh2016density,kingma2018glow}, probabilistic inference \cite{bauer2019resampled,ward2019improving,louizos2017multiplicative}, solving inverse problems \cite{ardizzone2018analyzing} and uncertainty estimation \cite{liu2020simple}. A common way to construct invertible networks is to compose invertible affine transformations with more sophisticated invertible layers, including coupling flows \cite{dinh2016density,kingma2018glow}, auto-regressive models \cite{huang2018neural,de2020block,ho2019flow},  invertible residual layers \cite{chen2019residual,behrmann2019invertible}, monotone networks \cite{ahn2022invertible}, and neural ordinary differential equations \cite{grathwohl2019scalable}, see also in the surveys \cite{papamakarios2021normalizing,kobyzev2020normalizing}. 

However, \cite{behrmann2021understanding} observed that commonly-used invertible networks suffer from exploding inverses and are thus prone to becoming numerically non-invertible. This observation motivates the input-output property of \emph{bi-Lipschitzness}. A layer $\gF:\R^n\rightarrow\R^n$ is said to be bi-Lipschitz with bound of $(\mu, \nu)$, or simply $(\mu, \nu)$-Lipschitz, if the following inequalities hold for all $x, x'\in \R^n$:
\[
\mu \|x-x'\|\leq \|\gF(x)-\gF(x')\|\leq \nu \|x-x'\|,
\]
where $\|\cdot\|$ is the Euclidean norm. The bound $\nu$ controls the output sensitivity to input perturbations while $\mu$ controls the input distinguishability from different outputs \cite{liu2020simple}. We call $\mu$ as the \emph{inverse Lipschitz} bound of $\gF$ as $\gF^{-1}$ exists and is $1/\mu$-Lipschitz. The ratio $\tau:=\nu/\mu$ is called \emph{distortion} \cite{liang2023low}, which is the upper bound of the condition number of the Jacobian matrix of  $\gF$. A larger distortion implies more expressive flexibility in the model. 

In this paper we argue that the bi-Lipschitz property is also useful for learning of surrogate loss (or reward) functions. Given some input/output pairs of a loss function, the objective is to learn a function which matches the observed data and is ``easy to optimize'' in some sense. This problem appears in many areas, including Q-learning with continuous action spaces, see e.g. \cite{gu2016continuous, amos2017input, ryu2019caql}, offline data-driven optimization \cite{grudzien2024functional}, learning reward models in inverse reinforcement learning \cite{arora2021survey}, and data-driven surrogate losses for engineering process optimization \cite{cozad2014learning, misener2023formulating}. An important contribution was the input convex neural network (ICNN) \cite{amos2017input}. However, the requirement of input convexity could be too strong in many applications.

\subsection{Contributions} 
\begin{itemize}
    \item We propose a novel strongly monotone and Lipschitz residual layer of the form $\gF(x)=\mu x+ \gH(x)$. For the nonlinear block $\gH$, we introduce a new architecture -- \emph{feed-through network} (FTN), which takes a multi-layer perceptron (MLP) as its backbone and adds connections from each hidden layer to the input and output variables. For deep networks, this architecture can improve the model expressivity without suffering from vanishing gradients. 
    
    \item We parameterize FTNs with certified  \emph{strong monotonicity} (which implies inverse Lipschitzness) and Lipschitzness for $\gF$ via the integral quadratic constraint (IQC) framework \cite{megretski1997system} and the Cayley transform. 

    \item By composing strongly-monotone and Lipschitz FTN layers with orthogonal affine layers we obtain the \textit{BiLipNet}, a new network architecture with smoothly-parameterized bi-Lipschitz bounds.
    
    \item We formulate the model inversion $\gF^{-1}$ as a three-operator splitting problem, which admits a numerically efficient solver \cite{davis2017three}.
    
    \item  We introduce a new scalar-output network $f:\R^n\rightarrow \R$, which we call a \textit{Polyak-\L{}ojasiewicz} network (or PLNet) since it satisfies the condition of the same name \cite{polyak1963gradient, lojasiewicz1963topological}. It consists of a bi-Lipschitz network composed with a quadratic potential, and automatically satisfies favourable properties for surrogate loss learning, in particular existence of a unique global optimum which is efficiently computable.
\end{itemize}

\subsection{Related work}

\paragraph{Bi-Lipschitz invertible layer.} In literature, there are two types of invertible layers closely related to our models. The first is the invertible residual layer $\gF(x)=x+\gH(x)$ \cite{chen2019residual,behrmann2019invertible}, where the nonlinear block $\gH$ is a shallow network with Lipschitz bound of $c < 1$. In \cite{perugachi2021invertible}, $\gH$ is further extended to a deep MLP. It is easy to show that $\gF$ is $(1-c)$-inverse Lipschitz and $(1+c)$-Lipschitz. In both cases, the Lipschitz regularization is via spectral normalization \cite{miyato2018spectral}, which we observe to be very conservative (see Figure \ref{fig:step}). Alternatively, a bi-Lipschitz layer can be defined by an implicit equation \cite{lu2021implicit, ahn2022invertible}. However, these require an iterative solver for both the forward and inverse model inference. In contrast, our model has an explicit forward pass and iterative solution is only required for the inverse. 

\paragraph{IQC-based Lipschitz estimation and training.} In \cite{fazlyab2019efficient}, the IQC framework of \cite{megretski1997system} was first applied to obtain accurate Lipschitz bound estimation of deep  networks with slope-restricted activations. It was later pointed out by \cite{wang2022quantitative} that IQC for Lipschitzness \cite{fazlyab2019efficient} is Shor's relaxation of a ``Rayleigh quotient" quadratically constrained quadratic programming (QCQP). Direct (i.e. unconstrained) parameterizations based on IQC were were proposed in \cite{revay2020lipschitz} for deep equilibrium networks, in \cite{araujo2023unified} for residual networks, for deep MLPs and CNNs in \cite{wang2023direct}, and recurrent models in \cite{revay2023recurrent}. It was pointed out by \cite{havens2023exploiting} that many recent Lipschitz model parameterizations \cite{meunier2022dynamical,prach2022almost,araujo2023unified,wang2023direct} are special cases of \cite{revay2020lipschitz}. In a recent work \cite{pauli2024novel}, the IQC-based Lipschitz estimation was recently extended to more general activations such as GroupSort and MaxMin. All of these are for one-sided (upper) Lipschitzness, whereas our work applies the IQC framework for monotonicity and bi-Lipschitzness.

\paragraph{Bi-Lipschitz networks for learning-based surrogate optimization.} \cite{liang2023low} uses Bi-Lipschitz networks to learn a surrogate constraint set while our work focuses on surrogate loss learning. Both works take distortion bound as an important regularization technique. The difference is that the distortion estimation in \cite{liang2023low} is based on data samples while our work offers certified and smoothly-parameterized distortion bounds.

\section{Preliminaries}

We give some definitions for a mapping $\gF:\R^n\rightarrow\R^n$.
\begin{definition}\label{defn:mon}
    $\gF$ is said to be {\bf $\mu$-strongly monotone} with $\mu>0$ if for all $x,x'\in \R^n$ we have
\begin{equation*}
    \inprod{\gF(x)-\gF(x')}{x-x'}\geq \mu \|x-x'\|^2,
\end{equation*}
    where $\inprod{\cdot}{\cdot}$ is the Euclidean inner product: $\inprod{a}{b}=a^\top b$. $\gF$ is {\bf monotone} if the above condition holds for $\mu=0$.  
\end{definition}

\begin{definition}
    $\gF$ is said to be {\bf $\nu$-Lipschitz} with $\nu>0$ if 
    \[
        \|\gF(x)-\gF(x')\|\leq \nu \|x-x'\|,\quad \forall x,x'\in \R^n.
    \]
    $\gF$ is said to be {\bf $\mu$-inverse Lipschitz} with $\mu>0$ if
    \[
         \|\gF(x)-\gF(x')\|\geq \mu \|x-x'\|,\quad \forall x_1,x_2\in \R^n.
    \]
    $\gF$ is said to be {\bf bi-Lipschitz} with $\nu\geq \mu>0$, or simply $(\mu,\nu)$-Lipschitz, if it is $\mu$-inverse Lipschitz and $\nu$-Lipschitz. 
\end{definition}
For any $(\mu, \nu)$-Lipschitz mapping $\gF$, its inverse $\gF^{-1}$ is well-defined and $(1/\nu,1/\mu)$-Lipschitz \cite{yeh2006real}. By the Cauchy–Schwarz inequality, strong monotonicity implies inverse Lipschitzness, see \cref{fig:mon-bilip}. A notable difference between monotonicity and bi-Lipschitzness is their composition behaviour. Given two bi-Lipschitz mappings $\gF_1, \gF_2$, their composition $\gF=\gF_2\circ \gF_1$ is also bi-Lipschitz with bound of $(\mu_1\mu_2, \nu_1\nu_2)$ where $(\mu_1,\nu_1)$ and $(\mu_2, \nu_2)$ are the bi-Lipschitz bounds of $\gF_1$ and $\gF_2$, respectively. However, given two strongly monotone $\gF_1,\gF_2$ with monotonicity bounds $\mu_1,\mu_2$, the composition $\gF=\gF_2\circ \gF_1$ does \textit{not} need to be strongly monotone. However, it is still $\mu_1\mu_2$-inverse Lipschitz. To quantify the flexibility of bi-Lipschitz maps, we introduce the following:

\begin{definition}
    $\gF$ satisfies a {\bf distortion bound $\tau$} with $\tau\geq 1$ if $\gF$ is $(\mu, \nu)$-Lipschitz with $\tau=\nu/\mu$. 
\end{definition}
For an invertible affine mapping $\gF(x)=Px+q$, the condition number of $P$ is a distortion bound. 
An orthogonal mapping (i.e., $P^\top P=I$) has the smallest possible model distortion $\tau=1$. Distortion bounds satisfy a composition property, i.e., if $\gF_1,\gF_2$ have distortion bounds of $\tau_1, \tau_2$, then $\gF_2\circ\gF_1$ satisfies a distortion bound of $\tau_1\tau_2$. Both $\gF$ and $\gF^{-1}$ have the same distortion.

\begin{figure}[!tb]
    \centering
    \includegraphics[width=0.7\linewidth]{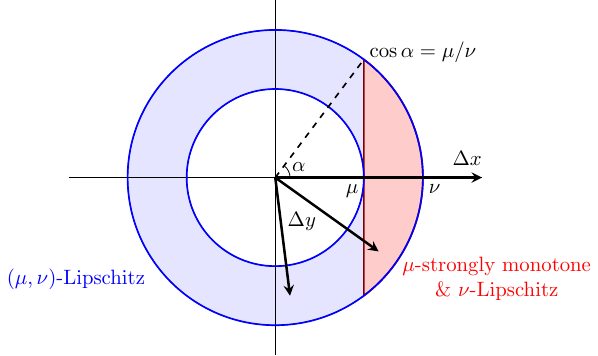}
        \caption{This figure depicts the possible ranges of $\Delta y=\gF(x')-\gF(x)$ on $\R^2$ for a given $\Delta x=x'-x $. The ring (blue area) is for $(\mu,\nu)$-Lipschitz $\gF$ while the half moon (red area) is for a $\mu$-strongly monotone and $\nu$-Lipschitz $\gF$. The largest angle between $\Delta x$ and $\Delta y$ satisfies $\cos\alpha=\tau^{-1}$ with $\tau=\nu/\mu$ as the distortion.}
    \label{fig:mon-bilip}
\end{figure}

\paragraph{Surrogate loss learning.} Let $\mathcal{D}$ be a dataset containing finite samples of $x_i\in \R^n$ and $y_i=\mathfrak{f}(x_i)\in \R$ where $\mathfrak{f}$ is an unknown loss function. The task is to learn a surrogate loss $\hat{f}$ from $\mathcal{D}$, i.e., $\hat{f}=\argmin_{f\in\mathfrak{F}}\,\mathbb{E}_{(x,y)\sim \mathcal{D}}\left[(f(x)-y)^2\right]$
where $\mathfrak{F}$ is the model set (e.g. neural networks). In many applications, it is highly desirable that each $f\in \mathfrak{F}$ has a unique and efficiently-computable global minimum. An important model class is the input convex neural network (ICNN) \cite{amos2017input}. Since $f\in\mathfrak{F}$ is convex w.r.t $x$, then any local minimum is a global minimum. Moreover, there exists a rich literature for convex optimization. Although convexity is more favourable for downstream optimization problems, it might be a very stringent requirement for fitting the dataset $\mathcal{D}$. In this work we aim to construct a model set $\mathfrak{F}$ such that every $f\in \mathfrak{F}$ does not need to be convex but still poses those favourable properties for optimization. In \cref{sec:PL}, we will show that the construction of such model set relies on bi-Lipschitz neural networks. 

\section{Monotone and bi-Lipschitz Networks}
In this section we first present the construction of $\mu$-strongly monotone and $\nu$-Lipschitz residual layers of the form $\gF(x)=\mu x+\gH(x)$. We then construct bi-Lipschitz networks by deep composition of the new monotone and Lipschitz layers with orthogonal linear layers.

\subsection{Feed-through network}

For the nonlinear block $\gH$, we introduce a network architecture, called \emph{feed-through network} (FTN), which takes an MLP as its backbone and then connects each hidden layer to input and output variables, see \cref{fig:ftn}. To be specific, the residual layer $\gF(x)=\mu x+ \gH(x)$ can be written as
\begin{equation}\label{eq:network}
    \begin{split}
        z_k&=\sigma(W_k z_{k-1}+ U_{k}x+b_k),\; z_0=0\\
        y &= \mu x + \sum_{k=1}^L Y_k z_k +b_y
    \end{split}
\end{equation}
where $z_k\in \R^{m_k}$ are the hidden variables, $U_k, W_k, Y_k$ and $b_k, b_y$ are the learnable weights and biases, respectively.  Throughout the paper we assume that the activation $\sigma$ is a scalar nonlinearity with slope restricted in $[0,1]$, which is satisfied (possibly with rescaling) by common activation functions such as ReLU, tanh, and sigmoid. 
\begin{remark}
    FTN contains both short paths $x\rightarrow z_i \rightarrow y$ preventing vanishing gradients and long paths $x\rightarrow z_i \rightarrow \cdots \rightarrow z_j \rightarrow y$ improving model expressivity (see \cref{fig:ftn}).
\end{remark}
\begin{figure}[!tb]
    \centering
    \includegraphics[width=\columnwidth]{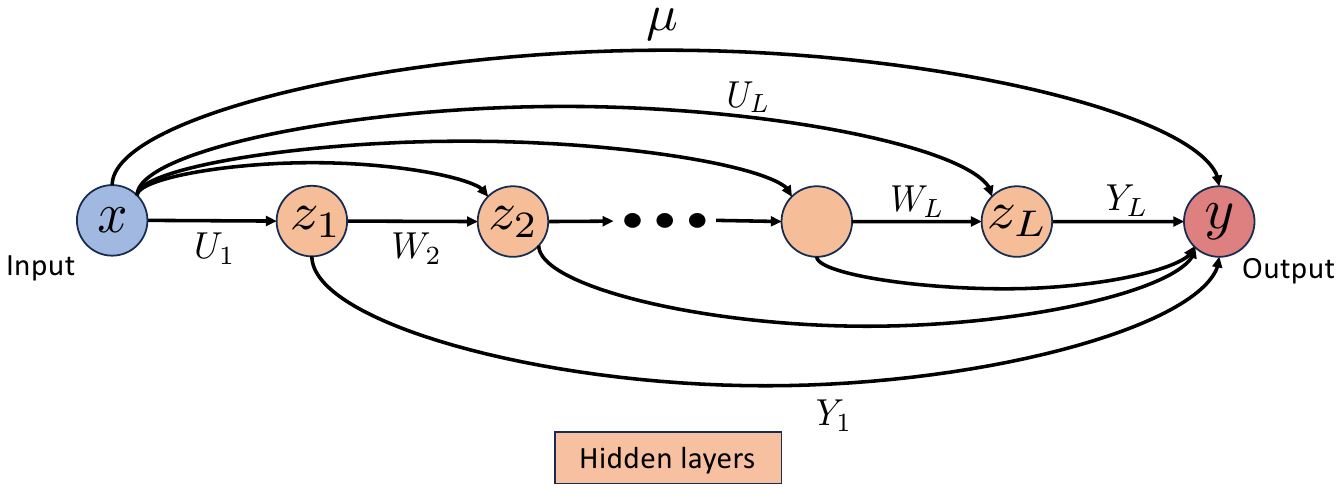}
    \caption{The proposed invertible residual network $\gF(x)=\mu x+\gH(x)$ where the nonlinear block $\gH$ is a feed-through network, whose hidden layers are directly connected to the input and output.}
    \label{fig:ftn}
\end{figure}

\subsection{SDP conditions for monotonicity and Lipschitzness}\label{sec:param-mln}   

The first step towards our parameterization is to establish strong monotonicity and Lipschitzness for $\gF$ via semidefinite programming (SDP) conditions. For this, we rewrite $\gF $ in a compact form:
\begin{equation}\label{eq:network-compact}
    z=\sigma(Wz+Ux+b),\quad y=\mu x+ Y z + b_y
\end{equation}
where $z=\bigl[\,z_1^\top\; \cdots\; z_L^\top\,\bigr]^\top,\,b=\bigl[\,b_1^\top\; \cdots\; b_L^\top\,\bigr]^\top $, and
\[
\begin{split}
    W&=\begin{bmatrix}
        0 & \\
        W_2 & 0 \\
        & \ddots & \ddots \\
        & & W_L & 0
    \end{bmatrix},\quad 
    U=\begin{bmatrix}
        U_1 \\  U_2 \\ \vdots \\ U_L
    \end{bmatrix}, \\
    Y&=\begin{bmatrix}
        Y_1 & Y_2 & \cdots & Y_L
    \end{bmatrix}.
\end{split}
\]

\begin{theorem}\label{thm:main}
    $\gF$ is $\mu$-strongly monotone and $\nu$-Lipschitz  if there exists a $\Lambda \in \sD_+^m$, where $\sD_+^m$ is the set of positive diagonal matrices, such that the following conditions hold: 
    \begin{equation} \label{eq:sdp}
        Y=U^\top \Lambda,\quad 2\Lambda-\Lambda W-W^\top \Lambda \succeq \frac{2}{\gamma} Y^\top Y  
    \end{equation}
    where $\gamma=\nu-\mu>0$.
\end{theorem}
\begin{remark}
    The above conditions are obtained by applying the IQC theory \cite{megretski1997system} to 
    %the implicit network representation 
    \eqref{eq:network-compact}. 
\end{remark}

\subsection{Model parameterization}\label{sec:model-param}

Let $\Theta$ be the set of all $\theta=\{U,W,Y,\Lambda \}$ such that Condition~\eqref{eq:sdp} holds. Since it is generally not scalable to train a model with SDP constraints, we instead construct a \textit{direct parameterization}, i.e. both unconstrained and complete:
\begin{definition}
A \textbf{direct parameterization}  of a constraint set $\Theta$ is a surjective differentiable mapping $\gM:\R^N\to \Theta$, i.e. for any $\phi \in \R^N$ we have $\gM(\phi)\in \Theta$, and the image of $\R^N$ maps onto $\Theta$, i.e. $\gM(\R^N)=\Theta$. 
 \end{definition}
A direct parameterization allows us to replace a constrained optimization over $\theta\in\Theta$ with an unconstrained optimization over $\phi\in\R^N$ without loss of generality. This enables use of standard first-order optimization algorithms such as SGD or ADAM \cite{kingma2014adam}.
 
We now construct a direct parameterization for FTNs satisfying \eqref{eq:sdp}. Here we present the main ideas, see \cref{sec:model-param-detail} for full details. First, we introduce the  free parameters
\[
    \phi=\{F^p, F^q\}\cup \{d_k, F_k^a, F_k^b\}_{1\leq k\leq L}
\]
where $F^p \in \R^{n \times n} $, $F^q\in \R^{m\times n}$, $d_k\in \R^{m_k}$, $F_k^a\in \R^{m_k\times m_k}$ and $F_k^b\in \R^{m_{k-1}\times m_k}$ with $m_0=0$. Then, we compute some intermediate variables $\Psi_k=\mathrm{diag}\bigl(e^{d_k}\bigl)$ and
\[
\begin{bmatrix}
    A_k^\top \\ B_k^\top 
\end{bmatrix}=\cayley\left(
\begin{bmatrix}
    F_k^a \\ F_k^b
\end{bmatrix}
\right),\;
\begin{bmatrix}
    P \\ Q
\end{bmatrix}=\cayley\left(
\begin{bmatrix}
    F^p \\ F^q
\end{bmatrix}
\right)
\]
where $\cayley:\R^{n\times p}\rightarrow\R^{n\times p}$ with $n\geq p$ is defined by 
\begin{equation}\label{eq:cayley}
    J=\cayley\left(\begin{bmatrix}
    G \\ H
\end{bmatrix}\right):=\begin{bmatrix}
    (I+Z)^{-1}(I-Z) \\ -2H(I+Z)^{-1}
\end{bmatrix}
\end{equation}
with $Z=G^\top  - G + H^\top H$. It can be verified that $J^\top J=I$ for any $G\in \R^{p\times p}$ and $H\in \R^{(n-p)\times p}$. Note that $P$ will not be used for further weight construction as its purpose is to ensure that $Q^\top Q \preceq I$. Next we set 
\[
V_k=2B_kA_{k-1}^\top,\quad S_k=A_kQ_k-B_kQ_{k-1}
\]
% $V_k=2B_kA_{k-1}^\top $ and $ S_k=A_kQ_k-B_kQ_{k-1}$
where $Q=\bigl[\,Q_1^\top\; \cdots\; Q_L^\top\,\bigr]^\top$ and $B_1=0$, $Q_0=0$. Finally, we construct the weights in \eqref{eq:network} as:
\begin{equation}\label{eq:direct-param}
    \begin{split}
        U_k&=\sqrt{2\gamma}\Psi_k^{-1}S_k,\; W_k=\Psi_k^{-1} V_k \Psi_{k-1},\\ 
        Y_k&=\sqrt{\frac{\gamma}{2}}S_k^\top \Psi_k,\; \Lambda_k=\frac{1}{2}\Psi_k^2.
    \end{split}
\end{equation}

\begin{proposition}\label{prop:direct-param}
    The model parameterization $\gM$ defined in \eqref{eq:direct-param} is a direct parameterization for the set $\Theta$, i.e. all models \eqref{eq:network} satisfying Condition~\eqref{eq:sdp}.
    % Parameterization \eqref{eq:direct-param} is a direct parameterization for the set $\Theta$, i.e. all models \eqref{eq:network-compact} satisfying \eqref{eq:sdp}.
\end{proposition}
This means that we can learn the free parameter $\phi$ using first-order methods without any loss of model expressivity.

The construction is now done, but we note that $\Psi_k$ is shared between layers $k$ and $k+1$. To have a modular implementation, we introduce new variables $\hat{z}=\Psi z$ and bias $\hat{b}=\Psi b$ with $\Psi=\diag(\Psi_1,\Psi_2,\ldots,\Psi_L)$. Then, \eqref{eq:network-compact} can be rewritten as follows (see \cref{sec:model-param-detail})
\begin{equation}\label{eq:new-form}   \hat{z}=\hat{\sigma}\bigl(V\hat{z}+\sqrt{2\gamma}S x+\hat b\bigr), \; y=\mu x+ \sqrt{\gamma/2}S^\top\hat{z}+b_y
\end{equation}
where $ \hat{\sigma}(x):=\Psi\sigma\left(\Psi^{-1} x\right)$ is a $(0,1)$-Lipschitz layer with learnable scaling $\Psi$, the weights $S,V$ can be written as
\begin{equation}\label{eq:SV}
    S=\begin{bmatrix}
    S_1 \\ S_2 \\ \vdots \\ S_L
\end{bmatrix}
,\quad V=\begin{bmatrix}
        0 & \\
        V_2 & 0 \\
        & \ddots & \ddots \\
        & & V_L & 0
    \end{bmatrix}.
\end{equation}
\paragraph{Number of free parameters.} Consider an $L$-layer FTN \eqref{eq:network} where each layer has the same width, i.e. $m_k=d$. The bi-Lipschitz network based on spectral normalization \cite{liu2020simple} has $2Ld^2$ free parameters while our model size is $(3L+1)d^2+Ld$.  Since the $Ld^2$ term dominates for deep and wide networks, our model has roughly 1.5 times as many parameter as the model from \cite{liu2020simple}. 

\subsection{Bi-Lipschitz networks}

We construct bi-Lipschitz networks (referred as BiLipNets) by composing strongly monotone and Lipschitz layers,
\begin{equation}\label{eq:bi-lip}
    \gG = \gO_{K+1}\circ \gF_{K}\circ \gO_{K}\circ \gF_{K-1}\circ \cdots \circ \gO_{2}\circ \gF_{1}\circ \gO_{1}
\end{equation}
where $\gO_k(x)=P_k x+q_k$ with $P_k^\top P_k=I$ is an orthogonal layer and $ \gF_k$ is a $\mu_k$-strongly monotone and $\nu_k$-Lipschitz layer \eqref{eq:new-form}. By the composition rule, the above BiLipNet is $(\mu,\nu)$-Lipschitz with $\mu=\prod_{k=1}^{L}\mu_k$ and $\nu=\prod_{k=1}^{L}\nu_k$. The orthogonal matrix $P$ can be parameterized via the Cayley transformation \eqref{eq:cayley} or Householder transformation \cite{singla2021improved}. Since the distortion of $\gO_k$ is 1, it can improve network expressivity without increasing model distortion. 

In some applications, e.g., normalising flows \cite{dinh2014nice,papamakarios2021normalizing}, we need to compute the inverse of $\gG$, which can be done in a backward manner:
\begin{equation}\label{eq:G-inv}
    \gG^{-1}(y)=\gO_1^{-1}\circ \gF_1^{-1}\circ \cdots \circ \gO_{K}^{-1}\circ \gF_K^{-1}\circ \gO_{K+1}^{-1}(y),
\end{equation}
where $\gO_k$ has an explicit inverse $\gO_k^{-1}(y)=P_k^\top (y-q_k)$.  Computing the inverse $\gF_k^{-1}(y)$ requires an iterative solver, which will be addressed in \cref{sec:inverse}.

\paragraph{Partially bi-Lipschitz networks.} A neural network $\tilde \gG:\R^n\times \R^l\rightarrow\R^n$ is said to be \emph{partially bi-Lipschitz} if for any fixed value of $p\in \R^l$, the mapping $y=\tilde \gG(x; p)$ is $(\mu,\nu)$-Lipschitz from $x$ to $y$. We can construct such mappings via $\tilde \gG(x;p)=\gG_{h(p)}(x)$ where $\gG_\phi$ is a$(\mu,\nu)$-Lipschitz network for any free parameter $\phi\in \R^N$ and $h:p\rightarrow \phi$ is a new learnable function. Since the dimension of $\phi$ is often very high, a practical approach is to make $\phi$ partially depend on $p$. For instance, we can learn $p$-dependent bias via an MLP while the weight matrices of $\gG_\phi$ is independent of $p$.

\section{Model inverse via operator splitting} \label{sec:inverse}

In this section we give an efficient algorithm to compute $\gF^{-1}(y)$ where $\gF$ is a $\mu$-strongly monotone and $\nu$-Lipschitz layer \eqref{eq:new-form}. First, we write its model inverse $\gF^{-1}$ as
\begin{equation}\label{eq:F-inverse}
    \begin{split}
        \hat z&=\hat \sigma\left(\left(V-\frac{\gamma}{\mu}SS^\top\right) \hat{z}+b_z\right) \\
        x&=\frac{1}{\mu}(y-b_y-\sqrt{\gamma/2}S^\top \hat z)
    \end{split}
\end{equation}
with $b_z=\sqrt{2\gamma}/\mu S(y-b_y)+\hat{b}$. Both $\gF$ and $\gF^{-1}$ can be treated as special cases of deep equilibrium networks \cite{bai2019deep,winston2020monotone,revay2020lipschitz} or implicit networks \cite{el2021implicit}. The difference is that $\gF$ has an explicit formula due to the strictly lower-triangular $V$ while $\gF^{-1}$ is an implicit equation as $SS^\top$ is a full matrix. A natural question for  \eqref{eq:F-inverse} is its \textbf{well-posedness}, i.e., for any $y\in \R^n$, does there exists a unique $\hat{z}\in \R^m$ satisfying \eqref{eq:F-inverse}?  
\begin{proposition}\label{prop:F-inverse}
    $\gF^{-1}$ is well-posed if $V,S$ are given by  \eqref{eq:SV}.
\end{proposition}

Certain classes of equilibrium networks were solved via two-operator splitting problems \cite{winston2020monotone,revay2020lipschitz}. We follow a similar strategy, but our structure admits a three-operator splitting, see \cref{prop:three-op} with background in \cref{sec:operator-split}. To state the result, we first recall the following fact from \cite{li2019lifted}. For the monotone and 1-Lipschitz activation $\hat{\sigma}$, there exists a proper convex function $f:\R^n\rightarrow\R$ satisfying $\hat \sigma(\cdot)=\prox_{f}^1(\cdot)$ with 
\[
\prox_{f}^\alpha(x)=\arg\min_{z\in\R^n}\; \frac{1}{2}\|x-z\|^2+\alpha f(z).
\]
A list of $f$ for popular activations is given in \cref{sec:mon-operator}. 
\begin{proposition}\label{prop:three-op}
    Finding a solution $\hat z\in \R^m$ to \eqref{eq:F-inverse} is equivalent to finding a zero to the three-operator splitting problem $0\in \gA(z)+\gB(z)+\gC(z)$ where $\gA,\gB,\gC$ are monotone operators defined by 
    \begin{equation*}
        \begin{split}
            \gA(z)=(I-V)z-b_z, \;
            \gB(z)=\partial f(z), \;
            \gC(z)= \frac{\gamma}{\mu} SS^\top z
        \end{split}
    \end{equation*}
    where $f$ satisfies $\hat{\sigma}(\cdot)=\prox_f^1(\cdot)$.
\end{proposition}
For three-operator problems, the Davis-Yin splitting algorithm (\textbf{DYS}) \cite{davis2017three} can be applied, obtaining the following fixed-point iteration:
\begin{equation}\label{eq:davis-yin}
    \begin{split}
        z^{k+1/2}&=\prox_f^\alpha (u^k) \\
        u^{k+1/2}&=2z^{k+1/2}-u^k \\ 
        z^{k+1}&= R_{\gA}(u^{k+1/2}-\alpha \gC(z^{k+1/2}))\\
        u^{k+1}&=u^k+z^{k+1}-z^{k+1/2}
    \end{split}
\end{equation}
where $R_{\gA}(v)=((1+\alpha)I-\alpha V)^{-1}(v+\alpha b_z)$. Since $V$ is strictly lower triangular, we can solve $ R_{\gA}(v)$ using forward substitution. Furthermore, we can show that \eqref{eq:davis-yin} is guaranteed to converge with $\alpha \in \bigl(0, \frac{1}{\tau -1}\bigr)$, where $\tau$ is the model distortion. 

\section{Polyak-\L{}ojasiewicz Networks}\label{sec:PL}
We call a network $f:\R^n\rightarrow\R$ a \emph{Polyak-\L{}ojasiewicz (PL) network}, or PLNet for short, if it satisfies the following PL condition \citep{polyak1963gradient,lojasiewicz1963topological}:
\begin{equation}
    \frac{1}{2}\|\nabla_x f(x)\|^2\geq m (f(x)-\min_x f(x)),\, \forall x\in \R^n,
\end{equation}
where $m>0$. The PL condition is significant in optimization since it is weaker than convexity, but still implies that gradient methods converge to a global minimum with a linear rate \citep{karimi2016linear}, making PLNet a promising candidate for learning a surrogate loss models.   
 
\begin{proposition}\label{prop:PL}
    If $\gG$ is $\mu$-inverse Lipschitz, then 
    \begin{equation}\label{eq:PL-f}
    f (x)=\frac{1}{2}\|\gG(x)\|^2+c,\quad c\in \R
    \end{equation}
    is a PLNet with $m=\mu^2$. 
\end{proposition}

\begin{remark}
    We can further relax the quadratic assumption: $f(x)=h\bigl(\gG(x)\bigr)$ is a PLNet if $h:\R^n\rightarrow \R$ is strongly convex \cite{karimi2016linear}. 
\end{remark}
\begin{remark}\label{rem:partialPL}
    For parametric optimization problem, one can learn a surrogate loss via $f(x;p)=1/2\|\tilde \gG(x;p)\|^2+c$
    where $p\in \R^m$ is the problem-specific parameter and $\tilde \gG$ is a partially bi-Lipschitz network.
\end{remark}
\begin{remark}
    Any sub-level set $\sL_\alpha=\{x: f(x)<\alpha\}$ with $\alpha>c$ is homeomorphic to a unit ball, making PLNets suitable for neural Lyapunov functions \cite{wilson1967structure}. Applications of PLNets to learning Lyapunov stable neural dynamics can be found in \cite{cheng2024learning}.
\end{remark}

\paragraph{Computing global optimum of a PLNet.} If $f$ takes the form \eqref{eq:PL-f} and $\gG$ is bi-Lipschitz network \eqref{eq:bi-lip}, then $f$ has a unique global optimum $x^\star=\gG^{-1}(0)$ with $\gG^{-1}$ given by \eqref{eq:G-inv}. This can be efficiently computed by analytical inversion of orthogonal layers and applying the DYS algorithm \eqref{eq:davis-yin} to monotone and Lipschitz layers. 

\paragraph{Limitations of gradient descent for finding global optimum.} 
An alternative way to compute the global optimum $x^\star$ is the standard gradient descent (GD) method $x^{k+1}=x^k-\alpha \nabla_x f(x^k)$. If $\nabla_x f$ is $L$-Lipschitz, then the above GD solver with $\alpha=1/L$ has a linear global convergence rate of $1-m/L$ with $m=\mu^2$ \citep{karimi2016linear}. However, this method has two drawbacks. First, the gradient function $\nabla_x f$ may not be globally Lipschitz, see \cref{exam:grad-lip}. Secondly, even if a global Lipschitz bound exists, it is generally hard to estimate.  
\begin{example}\label{exam:grad-lip}
    Consider a scalar function $f(x)=0.5g^2(x)$ with $g(x)=2x+\sin x$, which satisfies the PL condition. Note that $\partial f/\partial x=(2+\cos x)(2x+\sin x)$ is not globally Lipschitz due to the term $2x\cos x$. 
\end{example}

\section{Experiments}

Here we present experiments which explore the expressive quality of the proposed models, regularisation via model distortion, and performance of the DYS solution method. Code is available at \url{https://github.com/acfr/PLNet}.

\subsection{Uncertainty quantification via neural Gaussian process} 
It was shown in \cite{liu2020simple} that accurate uncertainty quantification of neural network models depends on a model's ability to quantify the distance of a test example from the training data. This \textit{distance-awareness} can be achieved with bi-Lipschitz residual layers $\gF(x)=x+\gH(x)$ and a Gaussian process output layer. In \cite{liu2020simple} this is achieved by imposing Lipschitz bound of $0<c<1$ for $\gH$ via spectral normalization. The resulting model is called Spectral-normalized Neural Gaussian Process (SNGP). In this section we examine the benefits of using the proposed BiLipNet in place of spectrally-normalized layers.

\paragraph{Toy example.} Using the two-moon dataset, we compare our $(\mu,\nu)$-Lipschitz network to an SNGP using a 3-layer i-ResNet under the same bi-Lipschitz constraints, i.e., $\mu=(1-c)^3$ and $\nu=(1+c)^3$, see \cref{fig:sngp}. For the lower-distortion case (i.e., small $c=0.1$), SNGP fails to completely separate the train and out-of-distribution (OOD) data due to its loose Lipschitz bound. Our model can distinguish the OOD examples from training dataset and the uncertainty surface is close to the SNGP with much higher distortion ($c=0.9$). As the model distortion increases, our model can have an uncertainty surface very close the dataset. The uncertainty surface of SNGP does not change much from $c=0.1$ to $c=0.9$, see Figure \ref{fig:sngp} and additional results in \cref{sec:extra-result}.

\begin{figure}[!tb]
    \centering
    \includegraphics[width=0.99\linewidth]{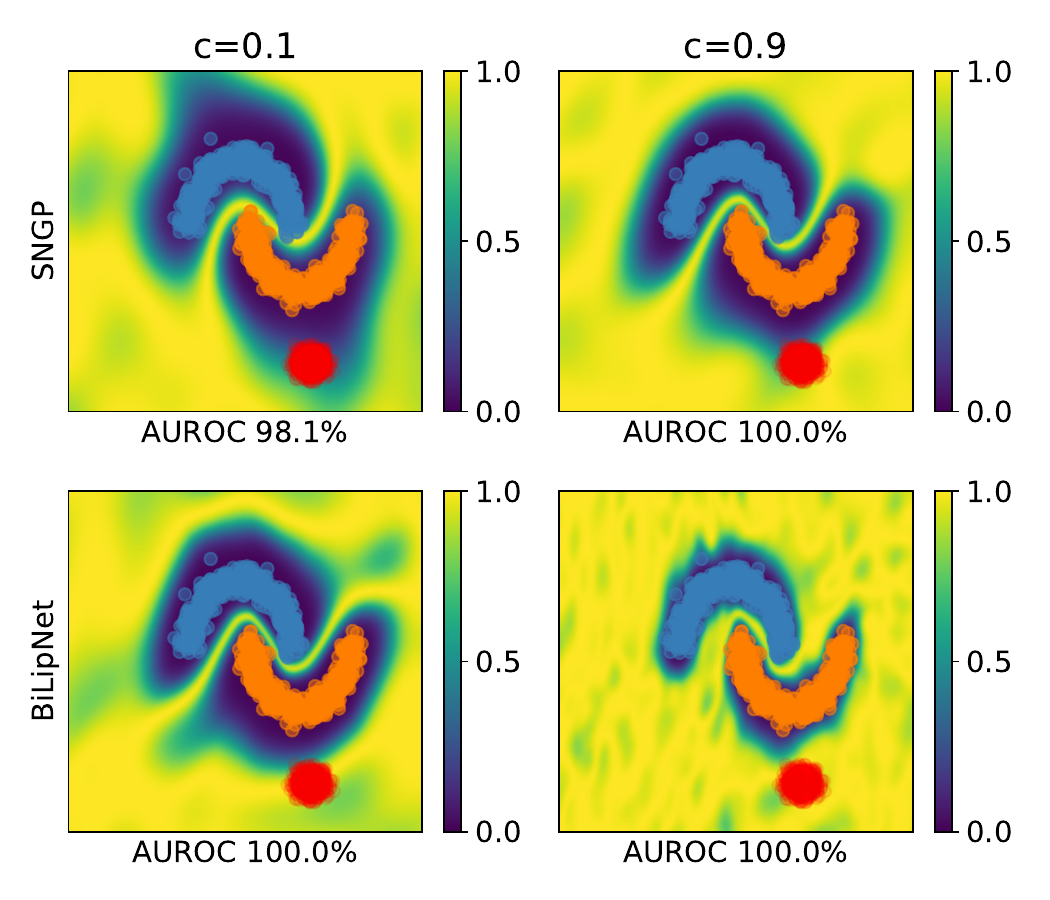}
    \caption{Predictive uncertainty of different NGPs with the same bi-Lipschitz bound. The points from dark blue and regions are classified as in-domain distribution and OOD data, respectively. Light blue and orange points (different colors indicate different labels) are training samples from the two-moon dataset. The red points are ODD test examples. For the case with small distortion, our model can still distinguish the train and OOD data, achieving similar results of SNGP with large distortion. }\label{fig:sngp}
\end{figure}

\paragraph{CIFAR-10/100.} For image datasets, the SNGP model in \cite{liu2020simple} contains three bi-Lipschitz components, each with four residual layers of the form $x+\gH(x)$ where $\gH$ is constructed to be $c$-Lipschitz using spectral normalization. To ensure certifiable bi-Lipschitzness, we modify the SNGP model by choosing $c\in (0,1)$ and removing batch normalization from $\gH$ since it may re-scale a layer’s spectral norm in unexpected ways \cite{liu2023simple}. The results of SNGP with batch normalization can be found in \cref{sec:extra-result}. Our BiLipNet model has a similar architecture as SNGP except replacing the bi-Lipschitz components with the proposed $(\mu,\nu)$-Lipschitz network \eqref{eq:bi-lip}. To ensure both models have the same bi-Lipschitz bound, we choose $\mu=(1-c)^4$ and $\nu=(1+c)^4$. 

\cref{tab:cifar10-100} reports the results of SNGP and BiLipNet under different bounds $c=0.95, 0.65, 0.35$. For CIFAR-10 dataset, our model uniformly outperforms SNGP on both clean and corrupted data, i.e., it achieves higher accuracy (about $10\sim 20\%$ improvement), lower expected calibration error (ECE) and negative log liklihood (NLL). Similar conclusion also holds for CIFAR-100 on accuracy and NLL, though our model has sightly higher ECE. 

As with the previous toy example, our model with small distortion ($\tau=18.6$ for $c=0.35$) achieves better accuracy than SNGP with large distortion ($\tau=2.3\times 10^{6}$ for $c=0.95$). Thus, we observe that our parameterization is much more expressive for a given distortion bound.

\begin{table*}[ht]
\centering
\resizebox{0.84\textwidth}{!}{  
\begin{tabular}{c|c|cc|cc|cc}
\toprule
& & \multicolumn{2}{c|}{Accuracy ($\uparrow$)} & 
\multicolumn{2}{c|}{ECE ($\downarrow$)} &
\multicolumn{2}{c}{NLL ($\downarrow$)} 
\\
Method & $c$  & Clean & Corrupted & Clean & Corrupted & Clean & Corrupted
\\
\midrule \midrule
\multicolumn{8}{c}{\textbf{CIFAR-10}} 
\\
\midrule
\multirow{3}{*}{SNGP}  & 0.95 & 76.7 $\pm$ 0.629 & 58.7 $\pm$ 1.000 & 0.057 $\pm$ 0.007 & 0.079 $\pm$ 0.006 & 0.682 $\pm$ 0.015 & 1.199 $\pm$ 0.041 \\
& 0.65 & 72.5 $\pm$ 1.500 & 54.7 $\pm$ 1.778 & 0.058 $\pm$ 0.006 & 0.078 $\pm$ 0.006 & 0.797 $\pm$ 0.046 & 1.303 $\pm$ 0.057 \\
& 0.35 & 62.7 $\pm$ 0.334 & 52.3 $\pm$ 0.721 & 0.069 $\pm$ 0.010 & 0.065 $\pm$ 0.006 & 1.055 $\pm$ 0.010 & 1.356 $\pm$ 0.018 \\
\midrule
\multirow{3}{*}{BiLipNet}  & 0.95 & 86.2 $\pm$ 0.250 & 70.8 $\pm$ 0.469 & 0.020 $\pm$ 0.003 & 0.052 $\pm$ 0.005 & 0.423 $\pm$ 0.006 & 0.895 $\pm$ 0.020 \\
& 0.65 & 86.7 $\pm$ 0.129 & 72.8 $\pm$ 0.592 & 0.015 $\pm$ 0.005 & 0.047 $\pm$ 0.009 & 0.400 $\pm$ 0.006 & 0.830 $\pm$ 0.024 \\
& 0.35 & 84.5 $\pm$ 0.184 & 72.6 $\pm$ 0.216 & 0.010 $\pm$ 0.002 & 0.052 $\pm$ 0.004 & 0.457 $\pm$ 0.002 & 0.827 $\pm$ 0.008 \\
\midrule \midrule
\multicolumn{8}{c}{\textbf{CIFAR-100}} 
\\
\midrule
\multirow{3}{*}{SNGP} & 0.95 & 36.9 $\pm$ 1.656 & 25.5 $\pm$ 1.406 & 0.131 $\pm$ 0.010 & 0.068 $\pm$ 0.005 & 2.493 $\pm$ 0.068 & 3.073 $\pm$ 0.069 \\
& 0.65 & 33.0 $\pm$ 0.481 & 24.3 $\pm$ 0.749 & 0.117 $\pm$ 0.006 & 0.068 $\pm$ 0.003 & 2.683 $\pm$ 0.015 & 3.140 $\pm$ 0.048 \\
& 0.35 & 26.5 $\pm$ 1.630 & 19.3 $\pm$ 1.296 & 0.101 $\pm$ 0.016 & 0.056 $\pm$ 0.010 & 3.020 $\pm$ 0.062 & 3.406 $\pm$ 0.073 \\
\midrule
\multirow{3}{*}{BiLipNet} & 0.95 & 51.0 $\pm$ 0.480 & 35.8 $\pm$ 0.397 & 0.230 $\pm$ 0.006 & 0.137 $\pm$ 0.007 & 2.064 $\pm$ 0.024 & 2.718 $\pm$ 0.014 \\
& 0.65 & 55.2 $\pm$ 0.426 & 39.2 $\pm$ 0.495 & 0.225 $\pm$ 0.004 & 0.137 $\pm$ 0.005 & 1.887 $\pm$ 0.021 & 2.576 $\pm$ 0.022 \\
& 0.35 & 54.4 $\pm$ 0.438 & 41.1 $\pm$ 0.200 & 0.194 $\pm$ 0.008 & 0.126 $\pm$ 0.009 & 1.876 $\pm$ 0.031 & 2.447 $\pm$ 0.016 \\
\bottomrule
\end{tabular}
}
\caption{ 
Results for SNGP and BiLipNet on CIFAR-10/100, averaged over 5 seeds. To ensure bi-Lipschitz bounds, batch normalization is removed from SNGP. BiLipNet uniformly significantly outperforms SNGP in term of accuracy on both clean and corrupted data. 
}
\label{tab:cifar10-100}
\end{table*}

\subsection{Surrogate loss learning} 
We explore the PLNet's performance with the Rosenbrock function $r(x,y)=1/200(x-1)^2+0.5(y-x^2)^2$ and its higher-dimensional generalizations. The  Rosenbrock function is a classical test problem for optimization, since it is non-convex but a unique global minimum point $(1,1)$, at which the Hessian is poorly-conditioned. We also consider the sum of the Rosenbrock function and a 2D sine wave function, which still has a unique global minimum at $(1,1)$ while having many local minima, see \cref{sec:train-detail}. 

We learned  models of the form \eqref{eq:PL-f} where $\gG$ is parameterized by MLP, i-ResNet \cite{behrmann2019invertible}, i-DenseNet \cite{perugachi2021invertible} and the proposed BiLipNet \eqref{eq:bi-lip}. We also trained the ICNN, a scalar-output model which is convex w.r.t. inputs \cite{amos2017input}. 

From \cref{fig:rosenbrock}, we have the following observations. The unconstrained MLP can achieve small test errors. However, it has many local minima near the valley $y=x^2$. This phenomena is more easily visible for the Rosenbrock+Sine case but also occurs in the plain Rosenbrock case. The ICNN model has a unique global minimum but the fitting error is large as its sub-level sets are convex. For i-DenseNet, the sub-level sets become mildly non-convex but their bi-Lipschitz bound is quite conservative, so they do not capture the overall shape. In contrast, our proposed BiLipNet is more flexible and captures the non-convex shape while maintaining a unique global minimum. We note that in the Rosenbrock+Sine case, the BiLipNet surrogate has errors of similar magnitude to the MLP, but remains ``easily optimizable'', i.e. it satisfies the PL condition and has a unique global minimum. Additional results are in \cref{sec:extra-result}.

\paragraph{Partial PLNet.} We also fit a parameterized Rosenbrock function $r(x,y; p)$ using partial PLNet with $p$-dependent biases (see Remark \ref{rem:partialPL}). The results in \cref{fig:rosenbrock2} indicate that the approach can be effective even if only bias terms are modified by the external parameter $p$, and not weights.

\paragraph{High-dimensional case.} We now turn to scalability of the approach to higher-dimensional problems and analyse convergence of the DYS method for computing the global minimum. We apply the approach to a $N$=20-dimensional version of the Rosenbrock function:
\begin{equation}\label{eq:multiRosen}
    R(x) = \frac{1}{N-1}\sum_{i=1}^{N-1}r(x_i, x_{i+1})
\end{equation}
which has a global minimum of zero at $x = (1,1, ... , 1)$ but is non-convex and has spurious local minima \cite{Kok09}. We sample 10K training points uniformly over $[-2,2]^{20}$. Note that, in contrast to the 2D example above, this is very sparse sampling of 20-dimensional space. 

A comparison of train and test error vs model distortion is shown in \cref{fig:rb20}. It can be seen that our proposed BiLipNet model achieves far better fits than iResNet \cite{behrmann2021understanding} and iDenseNet \cite{perugachi2021invertible}, which can not achieve small training error for any value of the distortion parameter. Furthermore, for our network, the distortion parameter appears to act as an effective regularizer. Note that the best test error occurs after training error drops to near zero ($\sim 10^{-8}$) but distortion is still relatively small. 

\begin{figure}[!tb]
    \centering
    \begin{tabular}{c}
        \includegraphics[width=\linewidth]{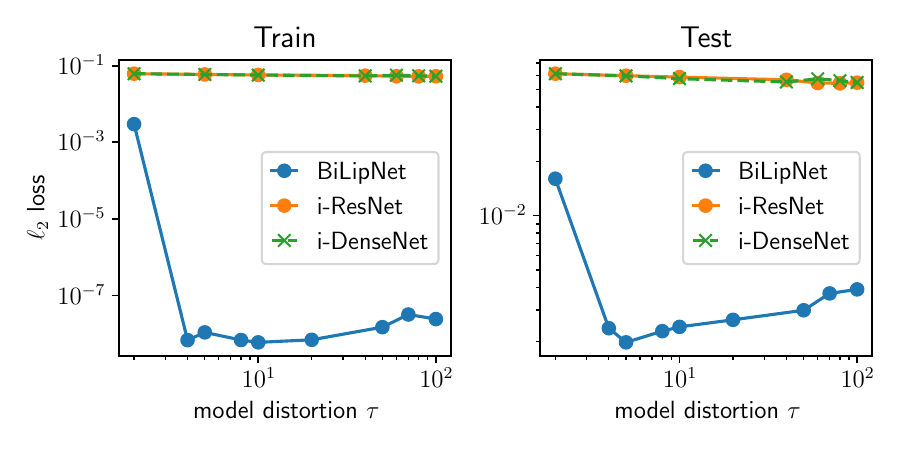} 
    \end{tabular}
    \caption{Surrogate loss learning for 20-dimensional Rosenbrock function. Comparison of training and test error vs model distortion for PLNet with  different bi-Lipschitz models.}
    \label{fig:rb20}
\end{figure}

\paragraph{Solver comparison.}
Given the surrogate loss function learned by BiLipNet, we now compare methods to compute the location of its global minimum. In \cref{fig:solver} we compare the proposed DYS solver to the forward step method (\textbf{FSM}), see, e.g., \cite{ryu2016primer}. Specifically, the inverse $x=\gF^{-1}(y)$ with $\gF$ as a $\mu$-strongly monotone and $\nu$-Lipschitz layer can be computed via 
\begin{equation}\label{eq:fwd}
    x^{k+1}=x^k-\alpha (\gF(x^k)-y)
\end{equation}
which has a convergence rate of $1-\mu^2/\nu^2$ if $\alpha=\mu/\nu^2$.
We also consider a commonly used gradient-based method -- ADAM \cite{kingma2014adam} applied directly to the surrogate loss. We take two values of the distortion parameter: $\tau = 5$ (optimal) and $\tau=50$. In both cases, the proposed DYS method converges significantly faster than the alternatives, and the results illustrate an additional benefit of regularising via distortion, besides improving the test error: the $\tau=5$ case converges significantly faster than $\tau=50$.

% Given the surrogate loss function learned by BiLipNet, we now compare methods to compute the location of its global minimum. In \cref{fig:solver} we compare the proposed DYS solver to the FSM algorithm \eqref{eq:fwd}
% and  ADAM \cite{kingma2014adam} applied directly to the surrogate loss. We take two values of the distortion parameter: $\tau = 5$ (optimal) and $\tau=50$. In both cases, the proposed DYS method converges significantly faster than the alternatives, and the results illustrate an additional benefit of regularising via distortion, besides improving the test error: the $\tau=5$ case converges significantly faster than $\tau=50$. 

At the computed  point $x^\star=\gG^{-1}(0)$ for $\tau=5$, the true function \eqref{eq:multiRosen} takes a value of $R(x^\star)=0.041$. This is more than an order of magnitude better than the smallest value of $R(x)$ over the training data, which ranged over $ [0.475, 6.532]$, indicating that PLNets have a useful ``implicit bias'' and do not simply interpolate the training data.

\begin{figure}
    \centering
    \includegraphics[width=\linewidth]{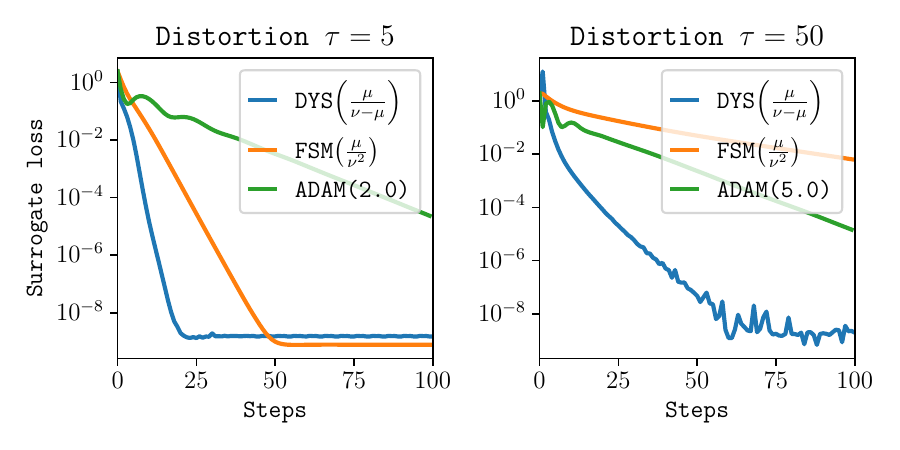}
    \caption{Solver comparison for finding the global minimum of a PLNet. We try a range of rates $[0.1, 0.5, 1.0, 2.0, 5.0]$ for ADAM and present the best result. The proposed back solve method with DYS algorithm \eqref{eq:davis-yin} converges much faster than ADAM applied to $f$ or back solve method with FSM algorithm \eqref{eq:fwd}. }
    \label{fig:solver}
\end{figure}

\begin{figure*}[!tb]
    \centering
    \begin{tabular}{c}         
    \includegraphics[width=0.82\linewidth]{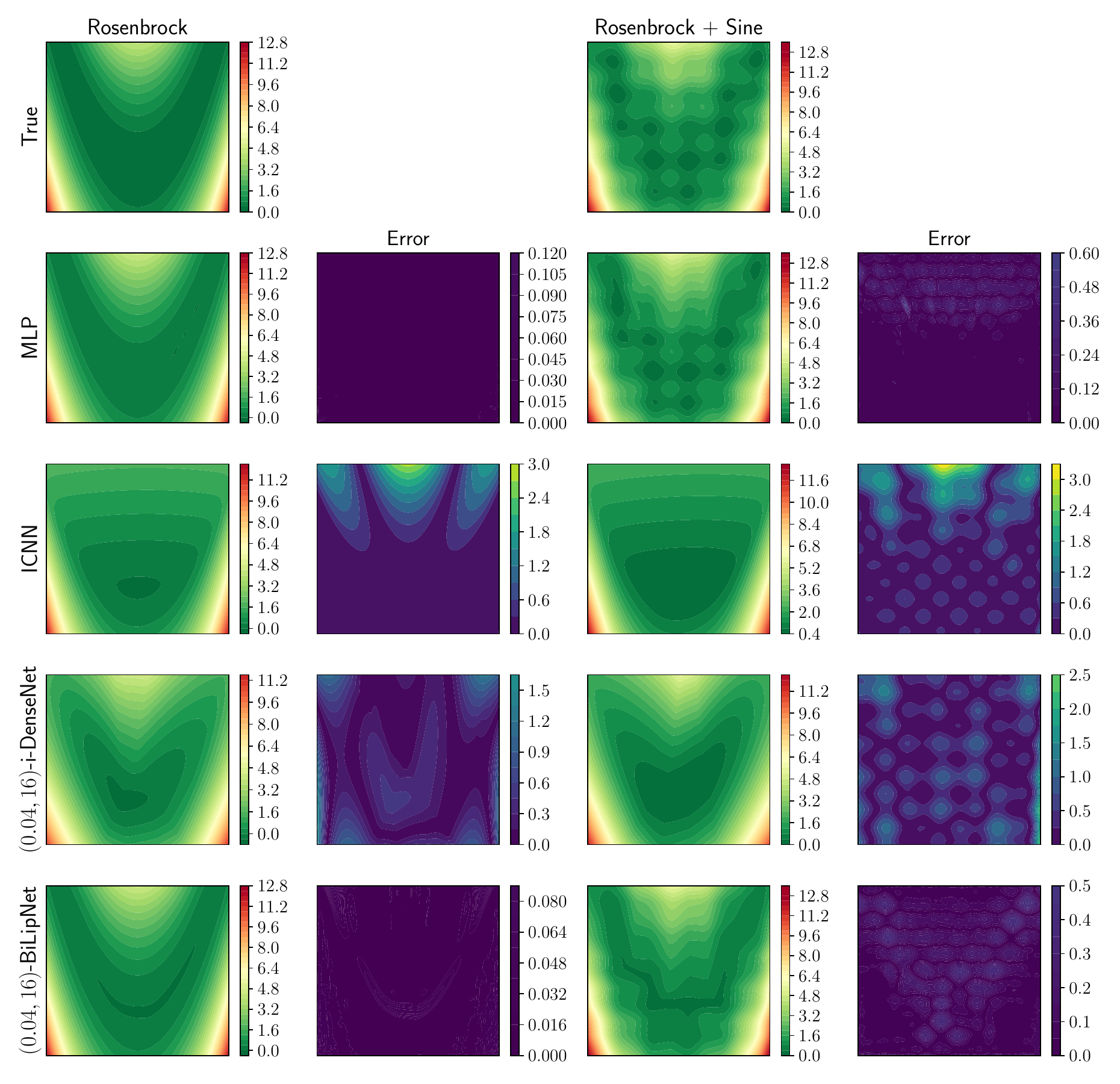} 
    \end{tabular}
    \caption{Learning a surrogate loss for the Rosenbrock and Rosenbrock+Sine functions, which is non-convex and has many local minima. The first row contains the true functions while the remaining rows show learned functions and errors for various surrogate loss models.}
    \label{fig:rosenbrock}
\end{figure*}

\begin{figure*}[!tb]
    \centering
    \begin{tabular}{c}
         \includegraphics[width=0.68\linewidth]{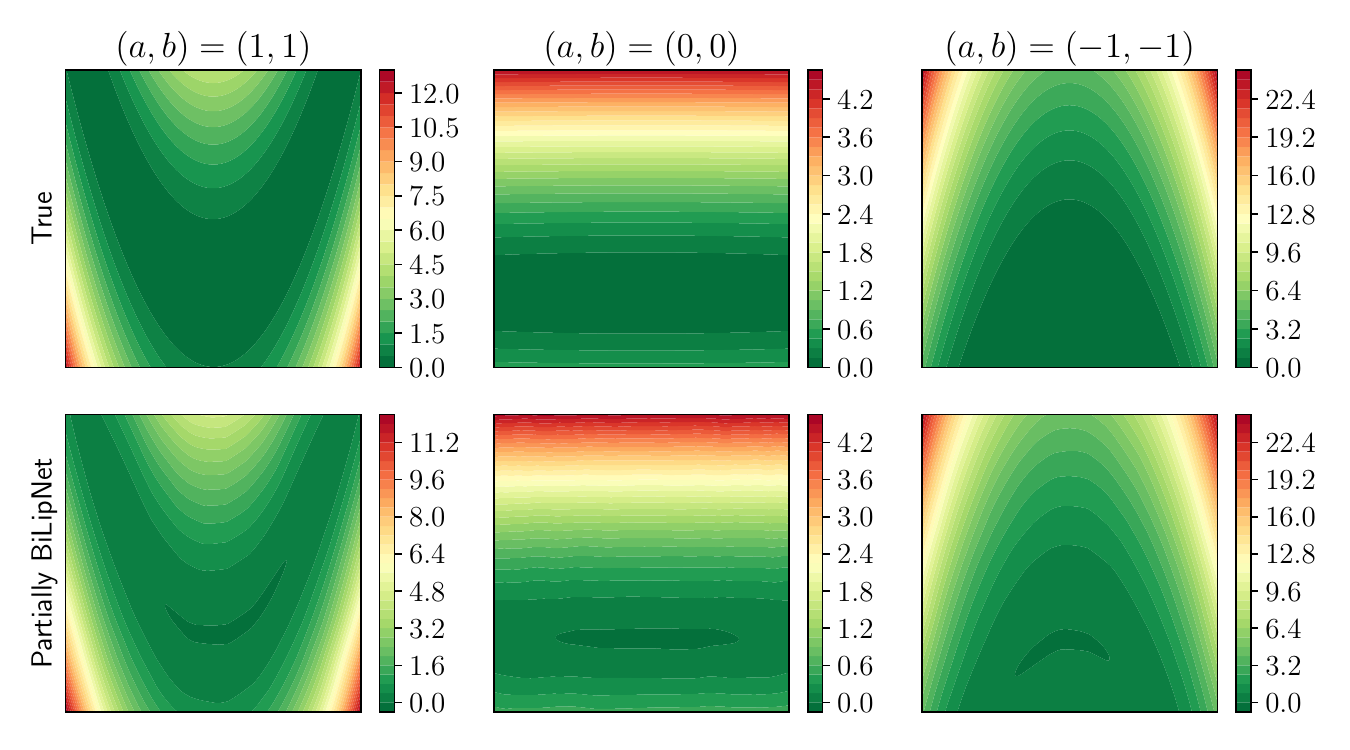}
    \end{tabular}
    \caption{Learning a parameterized Rosenbrock function $r(x,y; a,b)=1/200(x-a)^2+0.5(y-b x^2)^2$ via a partial PLNet.}
    \label{fig:rosenbrock2}
\end{figure*}

\section{Conclusion}
This paper has introduced a new bi-Lipschitz network architecture, the BiLipNet, and a new scalar-output network, the PLNet which satisfies the Polyak-\L{}ojasiewicz condition, and is hence ``easily optimizable''.

The core technical contribution is a new layer-type: the ``feed-through'' layer, which has certified bounds for strong monotonicity and Lipschitzness.
By composing  with orthogonal layers we obtain a bi-Lipschitz network structure (BiLipNet) which has much tighter bounds than existing bi-Lipschitz residual networks based on spectral normalization. The PLNet composes a BiLipNet with a quadratic output layer, and guarantees unique global minimum which is efficiently computable.

\clearpage
\section*{Impact Statement}

There are many application domains in which the trustworthiness of machine learning is a live topic of debate and raises important and challenging questions. The goal of this paper is to advance the sub-field of machine learning methods which have mathematically-certified properties. In particular, in this paper one application is uncertainty quantification. We hope that a positive impact of our paper and others like it will be to the development of ML methods that can better satisfy societal expectations of trustworthiness and transparency.

We are not aware of any potentially significant negative impacts that are particularly associated with this line of research (models with certified properties).


\bibliography{ref}
\bibliographystyle{icml2024}

%%%%%%%%%%%%%%%%%%%%%%%%%%%%%%%%%%%%%%%%%%%%%%%%%%%%%%%%%%%%%%%%%%%%%%%%%%%%%%%
%%%%%%%%%%%%%%%%%%%%%%%%%%%%%%%%%%%%%%%%%%%%%%%%%%%%%%%%%%%%%%%%%%%%%%%%%%%%%%%
% APPENDIX
%%%%%%%%%%%%%%%%%%%%%%%%%%%%%%%%%%%%%%%%%%%%%%%%%%%%%%%%%%%%%%%%%%%%%%%%%%%%%%%
%%%%%%%%%%%%%%%%%%%%%%%%%%%%%%%%%%%%%%%%%%%%%%%%%%%%%%%%%%%%%%%%%%%%%%%%%%%%%%%
\newpage
\appendix
\onecolumn

\section{Model Parameterization}\label{sec:model-param-detail}

A model parameterization is a mapping $\gM:\phi \rightarrow\theta $ where $\phi \in \R^N$ is a free learnable parameter while $\theta$ includes the model weights $U\in \R^{m\times n},W\in\R^{m\times m}, Y\in\R^{n\times m}$ and IQC multiplier $\Lambda\in \sD_+^m$ with $n,m$ as the dimensions of the input and hidden units, respectively. The aim of this section is to construct a parameterization such that the large-scale SDP constraint \eqref{eq:sdp} holds, i.e., $Y=U^\top \Lambda$ and
\begin{equation}\label{eq:H}
     \begin{split}
         H=2\Lambda -W^\top\Lambda-\Lambda W
         =\begin{bmatrix}
        2\Lambda_1 & - {W}_2^\top \Lambda_2 \\ 
         -  \Lambda_2 W_2  & 2\Lambda_2 & -{W}_3^\top \Lambda_3  \\
        & \ddots & \ddots & \ddots \\
        & & -\Lambda_{L-1} W_{L-1} & 2\Lambda_{L-1} & - W_L^\top \Lambda_L  \\
        & & & -\Lambda_L {W}_L & 2\Lambda_L
    \end{bmatrix}\geq \frac{2}{\gamma} Y^\top Y
     \end{split}
\end{equation}
Since $H\succeq 0$ has band structure, it can be represented by $H=X X^\top $ \cite{davis2006direct}. Moreover, from Lemma 3 of \cite{rantzer1996kalman} we have that any $U, Y$ satisfying $Y=U^\top \Lambda$ and $XX^\top \succeq \frac{2}{\gamma} Y^\top Y$ can be represent by
\begin{equation}\label{eq:UY}
    U=\sqrt{\gamma/2} \Lambda^{-1}X Q,\quad   Y=\sqrt{\gamma/2} Q^\top X^\top
\end{equation}
where $Q\in \R^{m\times n}$ with $QQ^\top \preceq I$. The remaining task is to find $X$ such that $H=XX^\top$ has the same sparse structure as \eqref{eq:H}, which was solved by \cite{wang2023direct}. For self-contained purpose, we provide detail construction as follows. First, we further parameterize $ X= \Psi P$, where $ \Psi=\mathrm{diag}(\Psi_1,\ldots,\Psi_L)$ with $\Psi_k\in \sD_+^{m_k}$ and 
\[
\begin{split}
P=
\begin{bmatrix}
    A_1 \\
    -B_2 & A_2 \\
    & \ddots & \ddots \\
    & & -B_L & A_L
\end{bmatrix}.
\end{split}
\]
By comparing $H=\Psi P P^\top \Psi$ with \eqref{eq:H} we have 
\begin{gather*}
    H_{kk}=\Psi_k(B_k B_k^\top + A_k A_k^\top)\Psi_k=2\Lambda_k, \quad
    H_{k-1,k}=-\Psi_{k} B_k A_{k-1}^\top =-\Lambda_k W_k,
\end{gather*}
which further leads to 
\begin{gather}
    \Psi_k^2=2\Lambda_k, \quad B_k B_k^\top + A_kA_k^\top = I,\quad W_k=2\Psi_k^{-1}B_kA_{k-1}^\top \Psi_{k-1} \quad k=1,\ldots,L,  \label{eq:PAB} 
\end{gather}
with $B_1=0$. We have converted the large-scale SDP constraint \eqref{eq:H} into many simple and small-scale constraints such as 
\begin{equation}\label{eq:sdp-simple}
    \Psi_k^2=2\Lambda_k, \quad R_k R_k^\top=I,\quad QQ^\top \preceq I
\end{equation}
with $R_k=\begin{bmatrix}
    B_k & A_k
\end{bmatrix}$, which further can be easily parameterized via the Cayley transformation \eqref{eq:cayley}, see \cref{sec:model-param}. The Cayley transformation has been applied to construct orthogonal layers \cite{helfrich2018orthogonal,li2019efficient,trockman2021orthogonalizing} and 1-Lipschitz Sandwich layer \cite{wang2023direct}. 

\paragraph{An equivalent model representation.} The model weights $U, Y, W$ defined in \eqref{eq:direct-param} can be rewritten as $U=\sqrt{2\gamma}\Psi^{-1}S$, $Y=\sqrt{\gamma/2}S^\top \Psi^{-1}$ and $W=\Psi^{-1} W \Psi$ with
\begin{equation}\label{eq:VS}
    S=\begin{bmatrix}
    S_1 \\ S_2 \\ \vdots \\ S_L
\end{bmatrix}=
\begin{bmatrix}
    A_1 Q_1 \\
    A_2Q_2-B_2Q_{1} \\
    \vdots \\
    A_LQ_L-B_LQ_{L-1}
\end{bmatrix},\quad V=\begin{bmatrix}
        0 & \\
        V_2 & 0 \\
        & \ddots & \ddots \\
        & & V_L & 0
    \end{bmatrix}=
    \begin{bmatrix}
        0 & \\
        2 B_2 A_{1}^\top & 0 \\
        & \ddots & \ddots \\
         & & 2 B_L A_{L-1}^\top & 0
    \end{bmatrix}
\end{equation}
where $Q=\begin{bmatrix}
    Q_1^\top & \cdots & Q_L^\top
\end{bmatrix}^\top$. Then, the network \eqref{eq:network-compact} can be written as 
\begin{equation}
    z=\sigma(\Psi^{-1}V\Psi z + \sqrt{2\gamma}\Psi^{-1} S x + b),\quad y= \mu x + \sqrt{\gamma/2}S^\top \Psi z.
\end{equation}
By introducing the new hidden state $\hat{z}=\Psi z$ and bias $\hat{b}=\Psi b$, we obtain an equivalent form:
\begin{equation}
    \hat{z}=\hat{\sigma}\bigl(V\hat{z}+\sqrt{2\gamma}S x+\hat b\bigr), \quad y=\mu x+ \sqrt{\gamma/2}S^\top\hat{z}+b_y.
\end{equation}
This representation is useful for computing the model inverse via monotone operator splitting, see \cref{sec:operator-split}. We now give a lemma which will be used later for proving some propositions.
\begin{lemma}\label{lemma:VS}
    For the matrices $V,S$ defined in \eqref{eq:VS} we have 
    \begin{equation}
        2I-V-V^\top \succeq 0,\quad 2I - S S^\top \succeq 0.
    \end{equation}
\end{lemma}
\begin{proof}
First, we have 
\begin{equation*}
    2I-(V+V^\top) =2\begin{bmatrix}
    I & -A_1B_2^\top \\
    -B_2A_1^\top & I & -A_2B_3^\top \\
    & -B_3^\top A_2 & \ddots & \ddots \\ 
    & & \ddots & \ddots  
\end{bmatrix}\succeq 0
\end{equation*}
where the inequality is obtained by sequentially applying the fact $A_kA_k^\top + B_kB_k^\top=I$ and Schur complement to the top diagonal block. For the inequality on $S$, we have 
\begin{equation*}
\begin{split}
    2I-S S^\top&= 2I - PQQ^\top P^\top \succeq 2I-PP^\top \\
    &=2I-\begin{bmatrix}
    A_1 \\
    -B_2 & A_2 \\
    & \ddots & \ddots \\
    & & -B_L & A_L
\end{bmatrix}
\begin{bmatrix}
    A_1 \\
    -B_2 & A_2 \\
    & \ddots & \ddots \\
    & & -B_L & A_L
\end{bmatrix}^\top=\begin{bmatrix}
    I & A_1B_2^\top \\
    B_2A_1^\top & I & A_2B_3^\top \\
    & B_3^\top A_2 & \ddots & \ddots \\ 
    & & \ddots & \ddots  
\end{bmatrix}\succeq 0.
\end{split}
\end{equation*}
Similarly, the last inequality can be established by sequentially applying the Schur complement to the top diagonal block.
\end{proof}

\section{Monotone Operator Splitting for Computing Model Inverse}\label{sec:operator-split}
Inspired by \cite{winston2020monotone,revay2020lipschitz}, we try to compute $x=\gF^{-1}(y)$ via an operator splitting method. We first present some background of monotone operator theory based on the survey \cite{ryu2016primer}, and then reformulate the model inverse as a three-operator splitting problem. 

\subsection{Monotone operator}\label{sec:mon-operator}

An {\em operator} is a set-valued or single-valued map defined by a subset of the space $\gA\subseteq \R^n\times\R^n$; we use the notation $\gA(x)=\{y\mid (x,y)\in \gA\}$. For example, the affine operator is defined by $ \gL(x)=\{(x,W x +b)\mid x\in\R^n\}$. Another important example is the subdifferential operator $\partial f=\{(x,\partial f(x))\}$ for a 
proper function $f:\R^n\rightarrow\R\cup \{\infty\}$ with $f(z)=\infty$ for $z\notin \dom f$,  where $\partial f(x)=\{g\in\R^n\mid f(y)\geq f(x)+\inprod{y-x}{g},\,\forall y\in\R^n\}$. An operator $\gA$ has a Lipschitz bound of $L$ if $ \|u-v\| \leq L\|x-y\| $ for all $(x,u),(y,v)\in \gA$. It is {\em non-expansive} if $L=1$ and {\em contractive} if $L<1$. $\gA$ is {\em strongly monotone} with $m>0$ if 
\begin{equation}
	\inprod{u-v}{x-y} \geq m\|x-y\|, \quad \forall (x,u),(y,v)\in \gA.
\end{equation}
If the above inequality holds for $m=0$, we call $\gA$ a monotone operator. Similarly, $\gA$ is said to be \emph{inverse monotone} with $\rho$ if $\inprod{u-v}{x-y} \geq \rho\|u-v\|, \ \forall (x,u),(y,v)\in \gA$. An operator is called {\em maximal monotone} if no other monotone operator strictly contains it. The linear operator $\gL$ is $m$-strongly monotone if $ W+W^\top \succeq 2mI$, and $\rho$-inverse monotone if $W+W^\top \succeq 2\rho W^\top W$. A subdifferential $\partial f$ is maximal monotone if and only if $f$ is a convex closed proper (CCP) function. Here are some basic operations for operators:
\begin{itemize}
    \item the operator sum $\gA+\gB=\{(x,y+z)\mid (x,y)\in \gA,\, (x,z)\in \gB\}$; 
    \item the composition $\gA \gB=\{(x,z)\mid \exists y\; \mathrm{s.t.}\; (x,y)\in\gA, (y,z)\in\gB\}$ ;
    \item the inverse operator $\gA^{-1}=\{(y,x)\mid (x,y)\in \gA\}$; 
    \item the {\em resolvent} operator $R_{\gA}=(I+\alpha \gA)^{-1}$ with $\alpha>0$;
    \item the {\em Cayley} operator $C_{\gA}=2R_{\gA}-I$.
\end{itemize}
Note that the resolvent and Cayley operators are non-expansive for any maximal monotone $\gA$, and are contractive if $\gA$ is strongly monotone. For a linear operator $\gL$ we have $R_{\gL}(x)=(I+\alpha W)^{-1}(x-\alpha b)$. For a subdifferential operator $\partial f$, its resolvent is $R_{\partial f}(x)=\prox_f^\alpha(x):=\argmin_{z}1/2\|x-z\|+\alpha f(z)$, which is also called the \emph{proximal operator}.

\paragraph{Activation as proximal operator.} As shown in \cite{li2019lifted,revay2020lipschitz}, many popular slope-restricted scalar activation functions can also be treated as proximal operators. To be specific, if $\sigma:\R\rightarrow\R$ is slope-restricted in $[0,1]$, then there exists a convex proper function $f$ such that $\sigma(\cdot)=\prox_f^1(\cdot)$. For self-contained purpose,  we provide a list of common activations and their associated convex proper functions in \cref{tab:sigma-f}, which can also be found in \cite{revay2020lipschitz,li2019lifted}.

\begin{table}[t]
\caption{A list of common activation functions and their associated convex proper $f(z)$ whose proximal operator is $\sigma(x)$ \cite{revay2020lipschitz}. For $z\notin \dom f$, we have $f(z)=\infty$. In the case of Softplus activation, $\mathrm{Li}_s(z)$ is the polylogarithm function.} \label{tab:sigma-f}
\begin{center}
    \begin{tabular}{|c|c|c|c|}
        \hline  &&& \\
        Activation & $\sigma(x)$ & Convex $f(z)$ & $\dom f$ \\ &&& \\ \hline &&& \\
        ReLu & $\max(x,0)$ & $0$ & $[0,\infty)$\\ &&& \\
        LeakyReLu &$\max(x,0.01x)$ & $\frac{99}{2}\min(z,0)^2$ & $\R$ \\ &&& \\
        Tanh &$ \tanh(x)$ & $\frac{1}{2}\left[\ln(1-z^2)+z\ln\left(\frac{1+z}{1-z}\right)-z^2\right]$ & $(-1,1)$ \\ &&& \\
        Sigmoid &  $1/(1+e^{-x})$ & $z\ln z+(1-z)\ln(1-z)-\frac{z^2}{2}$ & $(0,1)$\\ &&& \\
        Arctan & $\arctan(x)$ & $-\ln(|\cos z|)-\frac{z^2}{2}$ & $(-1,1)$\\ &&& \\
        Softplus & $\ln(1+e^x)$ & $ -\mathrm{Li}_2(e^z)-i\pi z-z^2/2 $ & $(0,\infty)$ \\ &&& \\ \hline
    \end{tabular}
\end{center}
\end{table}

\subsection{Operator splitting} 
Many optimization problems (e.g. convex optimization) can be formulated as one of finding a zero of an appropriate monotone operator $\gF$, i.e., find $x\in \R^n$ such that $0\in \gF(x)$. Note that $x$ is a solution if and only if it is a fixed point $x=\gT(x)$ with $\gT=I-\alpha \gF$ for any nonzero $\alpha \in \R$. The corresponding fixed point iteration is $x^{k+1}=\gT(x^k)=x^k-\alpha \gF(x^k)$. If $\gF$ is $m$-strongly monotone and $L$-Lipschitz, then this iteration converges by choosing $\alpha \in (0, 2m/L^2)$. The optimal convergence rate is $1-(m/L)^2$, given by $\alpha=m/L^2$. 

If $\gF$ contains some non-smooth components, we then split $\gF$ into two or three maximal operators:
\begin{align}
    \text{two-operator splitting problem:}\quad & 0\in \gA(x)+\gB(x) \\
    \text{three-operator splitting problem:}\quad & 0\in \gA(x)+\gB(x)+\gC(x)
\end{align}
where $\gA,\gB,$ and $\gC$ are maximal monotone. The main benefit of such splitting is that the resolvent or Cayley operators for individual operator are easy to evaluate, which further leads to more computationally efficient algorithms. For two-operator splitting problem, some popular algorithms include
\begin{itemize}
    \item forward-backward splitting (FBS) $x=R_{\gB} (I-\alpha \gA)(x)$
    \item forward-backward-forward splitting (FBFS) $x=((I-\alpha \gA) R_{\gB} (I-\alpha \gA)+\alpha \gA)(x)$
    \item Peaceman-Rachford splitting (PRS) $ z=C_{\gA} C_{\gB}(z),\; x=R_{\gB}(z)$
    \item Douglas-Rachford splitting (DRS) $ z=(1/2I+1/2C_{\gA} C_{\gB})(z),\; x=R_{\gB}(z)$
\end{itemize} 
where the corresponding fixed-point iterations, the choices of hyper-parameter $\alpha$ and convergence results can be found in \cite{ryu2016primer}. For three-operator splitting problem, the Davis-Yin splitting (DYS) \cite{davis2017three} can be expressed by $z=\gT(z),\, x=R_{\gB}(z)$ where $\gT=C_{\gA}(C_{\gB}-\alpha \gC R_{\gB})-\alpha \gC R_{\gB}$. 

\subsection{Operator splitting perspective for $\gF^{-1}$} \label{sec:pf-three-op}

As shown in the proof of \cref{prop:three-op}, By applying the forward-backward splitting with parameter $\alpha=1$, we can compute the solution $z$ via the following iteration:
\begin{equation*}
    \begin{split}
        z^{k+1}= R_{\gB}(z^k- \widehat{\gA}\bigl(z^k\bigr)) 
        =\hat \sigma\left(\left(V-\gamma/\mu SS^\top\right)z^k+b_z\right).
    \end{split}
\end{equation*}
It is worth pointing out that the above iteration may not converge for the choice of $\alpha=1$. In practice we often use more stable and faster two-operator splitting algorithms (e.g., PRS or DRS), see \cite{winston2020monotone,revay2020lipschitz}. In this work, the motivation for further decomposing the monotone operator $\widehat\gA$ into two monotone operators $\gA,\gC$ is that $R_{\gA}$ is a large-scale linear equation with nice sparse structure while $R_{\widehat \gA}$ is dense due to the full weight matrix in $\gC$.

\paragraph{Fixed-point iteration.} We now apply the DYS algorithm from \cite{davis2017three} to $0\in \gA(z)+\gB(z)+\gC(z)$, resulting in the following fixed-point iteration:
\begin{equation}\label{eq:DYS-FPI}
    \begin{split}
        z^{k+1/2}&=R_\gB (u^k)=\prox_f^\alpha (u^k) \\
        u^{k+1/2}&=2z^{k+1/2}-u^k \\ 
        z^{k+1}&=R_\gA(u^{k+1/2}-\alpha\gC(z^{k+1/2}))\\
        u^{k+1}&=u^k+z^{k+1}-z^{k+1/2}
    \end{split}
\end{equation}
where the third line is a large-scale sparse linear equation of the form
\[
\begin{bmatrix}
    (1+\alpha) I \\
    -\alpha V_{21} & (1+\alpha)I \\
    & \ddots & \ddots \\
    & & -\alpha V_{L,L-1} & (1+\alpha)I
\end{bmatrix}
\begin{bmatrix}
    z_1^{k+1} \\ z_2^{k+1} \\ \vdots \\ z_L^{k+1}
\end{bmatrix}=u^{k+1/2}+\alpha\left(b_z -\frac{\gamma}{\mu}SS^\top z^{k+1/2}\right).
\]

By introducing $v^{k+1/2}=b_z-\gamma/\mu SS^\top z^{k+1/2}$, we have
\begin{equation}
    z_0^{k+1}=0,\quad z_l^{k+1}=\frac{\alpha}{1+\alpha}\left(V_{l,l-1}z_{l-1}^{k+1}+v_{l}^{k+1/2}\right)+\frac{1}{1+\alpha}u_l^{k+1/2},\quad l=1,\ldots, L.
\end{equation}
\paragraph{Convergence range for the hyper-parameter $\alpha$.} From the previous paragraph, we know that \eqref{eq:davis-yin} is equivalent to the FPI \eqref{eq:DYS-FPI}. From Theorem~1.1 of \cite{davis2017three}, we have that \eqref{eq:DYS-FPI} converges for any $\alpha\in (0, 2\beta)$ with $\beta$ as the inverse-monotone bound of $\gC$. From \cref{lemma:VS} we have $2I\succeq S^\top S$ and
\[
\frac{2\gamma}{\mu} SS^\top \succeq \frac{\gamma}{\mu} S (S^\top S)S^\top = \frac{\mu}{\gamma} (\gamma/\mu SS^\top)^2=2\beta (\gamma/\mu SS^\top)^2
\]
i.e., $\gC(z)$ is inverse monotone with  $\beta=\mu/(2\gamma)$. Therefore, \eqref{eq:davis-yin} converges for any $\alpha \in (0, \mu/\gamma)$. Since $\gamma=\nu-\mu$ and $\tau=\nu/\mu$, we then obtain the convergence range in term of model distortion $\tau$, i.e., $\alpha \in (0, 1/(\tau-1))$. Larger $\alpha$ often implies faster convergence rate, see \cref{fig:dys}.

\begin{figure}[!tb]
    \centering
    \includegraphics[width=\linewidth]{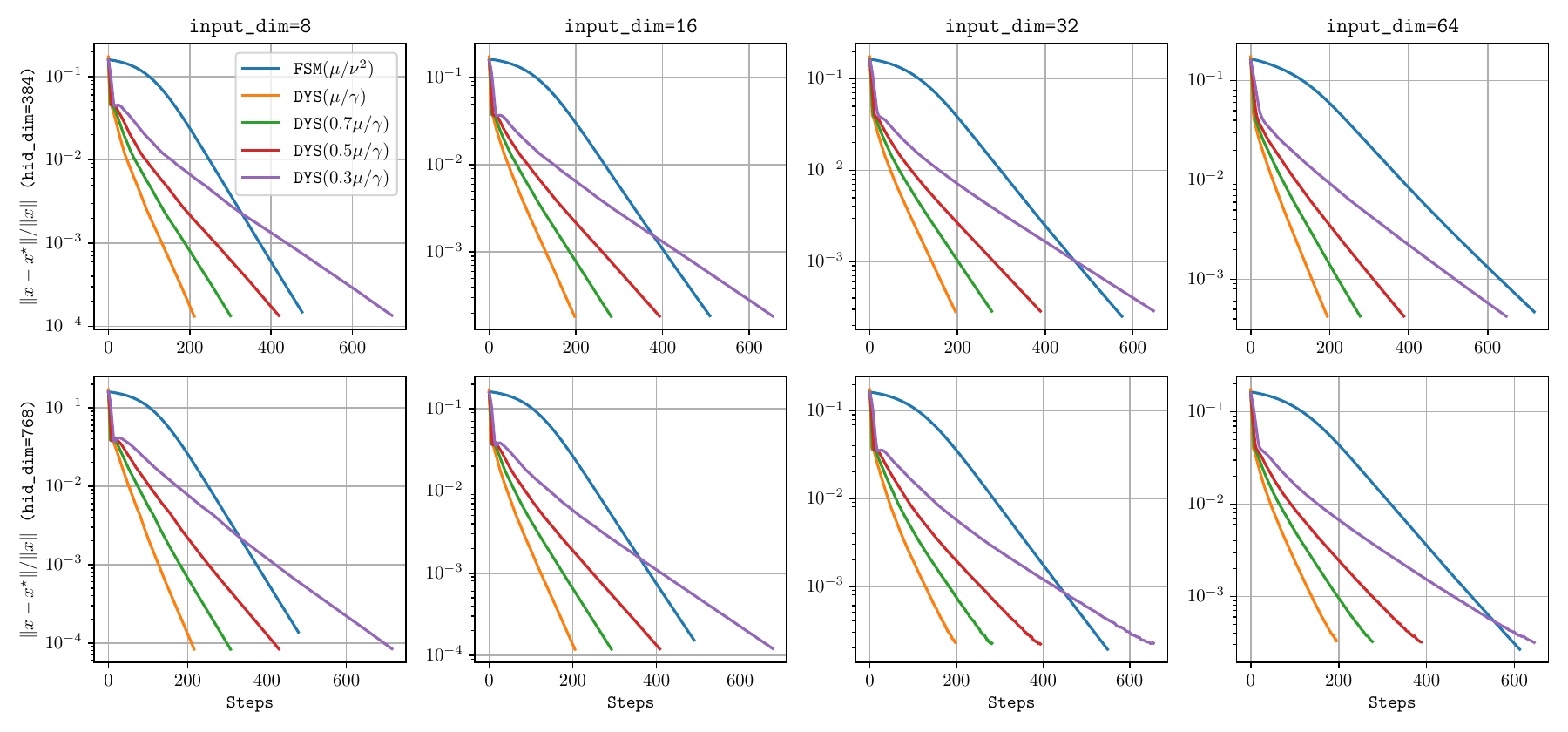}
    \caption{Solver comparison for computing the inverse of random $\mu$-monotone and $\nu$-Lipschitz layers with different input and hidden unit dimensions. We obverse that DYS \eqref{eq:DYS-FPI} converges faster for lager $\alpha$. If $\alpha$ is close to the bound $\mu/\gamma$ with $\gamma=\nu-\mu$, DYS converges much faster rate than FSM \eqref{eq:fwd} with hyper-parameter $\alpha=\mu/\nu^2$, which achieves its best convergence rate \cite{ryu2016primer}.}
    \label{fig:dys}
\end{figure}

\section{Proofs}
\subsection{Proof of \cref{thm:main}}\label{sec:pf-thm1}
We consider the neural network $\gH:x\rightarrow \tilde{y}$ defined by
\begin{equation}
    v=Wz+Ux+b,\quad z=\sigma(v),\quad \tilde y=Yz+b_y.
\end{equation}
Since $\gF(x)=\mu x+\gH(x)$, then $\gF$ is $\mu$-strongly monotone and $\nu$-Lipschitz if $\gH$ is monotone and $\gamma$-Lipschitz with $\gamma=\nu-\mu$.

For any pair of solutions $s_1=(x_1,v_1,z_1,\tilde{y}_1)$ and $s_2=(x_2,v_2,z_2,\tilde{y}_2)$, their difference $\Delta s=s_1-s_2$ satisfies
\begin{equation}\label{eq:difference}
    \Delta v= W\Delta z+ U\Delta x,\quad \Delta z=J_\sigma(v_1,v_2) \Delta v,\quad \Delta \tilde{y}=Y\Delta z
\end{equation}
where $J_\sigma$ is a diagonal matrix with $[J_{\sigma}]_{ii}\in [0,1]$ since $\sigma$ is an elementwise activation with slope restricted in $[0,1]$. For any $\Lambda\in \mathbb{D}_+^m$ we have
\begin{equation}\label{eq:iqc}
    \inprod{\Delta v - \Delta z}{\Lambda \Delta z}=\Delta v^\top (I-J_\sigma)\Lambda J_\sigma \Delta_v \geq 0,\quad \forall \Delta v\in \R^m.
\end{equation} 
Based on \eqref{eq:difference}, \eqref{eq:iqc} and Condition \eqref{eq:sdp} we have
\begin{equation*}
    \begin{split}
        \inprod{\Delta x}{\Delta \tilde y}-\inprod{\Delta v - \Delta z}{\Lambda \Delta z} 
        =&\inprod{\Delta x}{Y\Delta z}-\inprod{(W-I)\Delta z+U\Delta x}{\Lambda \Delta z} \\
        =&\inprod{\Delta x}{Y\Delta z}-\inprod{\Delta x}{U^\top \Lambda \Delta z}+ \inprod{(I-W)\Delta z}{\Lambda \Delta z} \\
        =&\frac{1}{2}\Delta z^\top \left(\Lambda(I-W)+(I-W^\top)\Lambda\right)\Delta z\geq \|Y\Delta z\|^2\geq 0,
    \end{split}
\end{equation*}
which further implies $ \inprod{\Delta x}{\Delta y}-\mu\|\Delta x\|^2 \geq \inprod{\Delta v - \Delta z}{\Lambda \Delta z}\geq 0$. Thus, $\gH$ is monotone. We can use the similar technique to derive the Lipschitz bound of $\gH$. Firstly we have
\begin{equation*}
    \begin{split}
        \gamma\|\Delta x\|^2-\frac{1}{\gamma}\|\Delta \tilde y\|^2-2\inprod{\Delta v - \Delta z}{\Lambda \Delta z} 
        =& \gamma\|\Delta x\|^2-\frac{1}{\gamma}\|\Delta \tilde y\|^2 + 2\inprod{(I-W)\Delta z}{\Lambda\Delta z}-2\inprod{U\Delta x}{\Lambda \Delta z}  \\
        =& \gamma\|\Delta x\|^2-2\inprod{\Delta x}{\Delta \tilde y}-\frac{1}{\gamma}\|\Delta \tilde y\|^2 + \Delta z^\top (2\Lambda -\Lambda W-W^\top \Lambda)\Delta z \\
        \geq & \gamma\|\Delta x\|^2-2\inprod{\Delta x}{\Delta \tilde y}-\frac{1}{\gamma}\|\Delta \tilde y\|^2 +\frac{2}{\gamma}\|Y\Delta z\|^2
        = \left\|\sqrt{\gamma}\Delta x-\frac{1}{\sqrt{\gamma}}\Delta \tilde y \right\|^2.
    \end{split}
\end{equation*}
Due to \eqref{eq:iqc} we can further obtain $\gamma^2\|\Delta x\|^2\geq \|\Delta \tilde y\|^2$, i.e., $\gH$ is $\gamma$-Lipschitz.

\subsection{Proof of \cref{prop:direct-param}}

\textit{Sufficient part: \eqref{eq:direct-param} $\Rightarrow$ \eqref{eq:sdp}.} From \eqref{eq:UY} we have that $Y=U^\top \Lambda$. We check the inequality part of \eqref{eq:sdp} as follows:
\[
\begin{split}
    2\Lambda - W^\top \Lambda-\Lambda W= \Psi P P^\top \Psi =X X^\top\succeq X Q Q^\top X^\top = \frac{2}{\gamma} Y^\top Y.
\end{split}
\]

\textit{Necessary part: \eqref{eq:sdp} $\Rightarrow$ \eqref{eq:direct-param}.} Since $H\succeq 0$ has band structure, then it can be decomposed into $H=XX^\top$  where $X$ has the following block lower triangular structure \cite{davis2006direct}:
\begin{equation*}
    X=\begin{bmatrix}
        X_{11} \\
        X_{21} & X_{22} \\
        & \ddots & \ddots \\
        & & X_{L,L-1} & X_{LL}
    \end{bmatrix}.
\end{equation*}
For this special case, a way to construct $X$ from $\Lambda, W$ and further computation of the free parameters $d, F_k^a, F_k^b, F^q, F^\star$ can be found in \cite{wang2023direct}. Finally, we need to show that $ X X^\top \succeq 2/\gamma Y^\top Y$ is equivalent to $ Y=\sqrt{\gamma/2}Q^\top X^\top$ for some $QQ^\top \preceq I$, which can be directly followed by Lemma 3 of \cite{rantzer1996kalman}.

\subsection{Proof of \cref{prop:F-inverse}}
From \cref{lemma:VS} we have
\begin{equation}
    2I-(V-\gamma/\mu SS^\top)-(V-\gamma/\mu SS^\top)^\top=2I-V-V^\top+2\gamma/\mu SS^\top \succeq 2\gamma/\mu SS^\top \succeq 0.
\end{equation}
Then, the equilibrium network \eqref{eq:F-inverse} is well-posed by Theorem~1 of \cite{revay2020lipschitz}.

\subsection{Proof of \cref{prop:three-op}} 
We first show that $0\in \gA(z)+\gB(z)+\gC(z)$ is a monotone operator splitting problem. It is obvious that $\gB,\gC$ are maximal monotone operators.  From \cref{lemma:VS} we have $(I-V)+(I-V)^\top \succeq 0$, i.e. $\gA$ is also monotone. Then, we show that the above operator splitting problem  shares the same set of equilibrium points with the model inverse \eqref{eq:F-inverse}. First, we rewrite it into a two-operator splitting problem $0\in \widehat{\gA}(z) + \gB(z)$ where $\widehat{\gA}=\gA+\gC$. By applying the forward-backward splitting with parameter $\alpha=1$, we can compute the solution $z$ via the following iteration:
\begin{equation*}
    \begin{split}
        z^{k+1}=& R_{\gB}(z^k- \widehat{\gA}\bigl(z^k\bigr)) 
        =\prox_{f}^1\left(z^k-\left(I-V+\gamma/\mu SS^\top\right)z^k+b_z\right) 
        =\hat \sigma\left(\left(V-\gamma/\mu SS^\top\right)z^k+b_z\right).
    \end{split}
\end{equation*}
Thus, any solution $z^\star$ of the equilibrium network \eqref{eq:F-inverse} is also an equilibrium point of the above iteration. 

\subsection{Proof of \cref{prop:PL}}

First, we have $\nabla f(x)=G^\top(x) \gG(x)$ where $G(x)= \nabla \gG(x)$ satisfies $\|G(x)\|\geq \mu$. Then, the PL inequality holds for $f$ with $m= \mu^2$, i.e.,
    \begin{equation}
        \frac{1}{2}\|\nabla f(x)\|^2=\frac{1}{2} \gG(x)^\top G(x)^\top G(x)\gG(x)\geq \frac{\mu^2 }{2} \|\gG(x)\|^2=\mu^2 (f(x)-f^\star).
    \end{equation}

\section{Experiments}
\subsection{Training details}\label{sec:train-detail}
    We choose ReLU as our default activation and use ADAM \cite{kingma2014adam} with one-cycle linear learning rate \cite{coleman2017dawnbench} except the NGP case which SGD with piecewise constant scheduling. For the NGP case, we use the cross entropy loss while the L2 loss is used for the rest of the examples. We found that it can improve the model training by enforcing $Q^\top Q=I$, which can be done by fixing $F^p=0$. Dataset and model architectures are described as follows.

\paragraph{1D Step function.} The target function is a step function 
\[
f(x)=\begin{cases}
    2,& x>0 \\
    -2, & x<0
\end{cases}
\]
which is monotone and $0$-Lipschitz everywhere except the singularity point $x=0$. We try to fit this curve with $(0.1, 10)$-Lipschitz models. The optimal fit is a linear piecewise continuous function with slope of 10 near $x=0$ and slope of 0.1 near $x=\pm 2$. We take 1000 random samples from $[-2,2]$ for training. Our model (BiLipNet) is an one-layer residual network $\gF(x)=\mu x+ \gH(x)$ where $\gH$ has 8 hidden layers of width 32, giving the model 15.8K parameters. We compare to i-ResNet \cite{chen2019residual} and i-DenseNet \cite{perugachi2021invertible}, where the nonlinear block $\gH$ has 2 and 4 hidden layers, respectively. For those two models, we test for depth from 2 to 8 with proper hidden width (so that they has similar amount of parameters).  And the empirical Lipschitz bound is computed via finite difference over the test data. As shown in \cref{fig:step}, our model achieves much tighter bounds than other models.

\paragraph{Neural Gaussian process.} We take 1000 two-moon data points as training data and 1000 Gaussian samples with mean $(1.3, -1.8)$ and variance $(0.02, 0.01)$ as OOD data. For all models, we use fixed input weight to mapping the 2D input into 128D hidden space, then perform hidden space transformation using bi-Lipschitz models, and finally add a Gaussian process as the output layer. SNGP uses 3 residual layer $x+\gH(x)$ where the Lipschitz bound of $\gH$ is $c<1$. BiLipNet has one monotone and Lipschitz layer with two orthogonal layer, i.e., $K=1$ for \eqref{eq:bi-lip}. The nonlinear block $\gH$ of our model has 6 hidden layers with width of 32. Both models are chosen to have the same amount of parameters, roughly $233K$. 

\paragraph{CIFAR-10/100 datasets.} We first adopt the SNGP model from \cite{liu2020simple} and make some modifications as follows. 
\begin{itemize}
    \item SNGP contains three bi-Lipschitz components with each including four residual layers of the form $x+\gH(x)$. It used spectral norm bound $c=6$ for the weights inside $\gH$, which means that the bi-Lipschitz property may not hold. To provide a certified guarantee of bi-Lipschitzness we need $c\in (0,1)$. We tried three values of $c$: 0.35, 0.65 and 0.95. Since the Lipschitz bounds are $\mu=(1-c)^4$ and $\nu=(1+c)^4$, a larger $c$ implies a more expressive SNGP model.
    \item We ran the SNGP with/without batch normalization for the bi-Lipschitz components. As pointed out in \cite{liu2023simple}, the batch normalization may re-scale a layer’s spectral norm in unexpected ways. So there is no theoretical guarantee on bi-Lipschitz property when batch normalization is applied.
    \item Training the original SNGP takes about 95\% GPU memory of an Nvidia RTX3090. With the same number of parameters, our model needs more GPU memory as it uses the approach from \cite{trockman2021orthogonalizing} to perform the Cayley transform of convolution operators, which involves FFT and inverse FFT. In order to use a single GPU to train both models, we reduce the width of SNGP so that it has a similar amount of parameters as our model ($\sim 14$M).
\end{itemize}
Our model has a similar structure to SNGP except that we replace their bi-Lipschitz components with our proposed bi-Lipschitz networks. Note that there is no batch normalization inside our bi-Lipschitz networks. All models are trained for 200 epochs using the mini-batch stochastic gradient descent (SGD) method with batch size of 256. We adjust the learning rate based on a piecewise constant schedule. 

\paragraph{2D Rosenbrock function.} The true function is a Rosenbrock function defined by 
\[
r(x,y)=\frac{1}{200}(x-1)^2+\frac{1}{2}\bigl(y-x^2\bigr)^2.
\]
Note that we use a scaling factor of $1/200$ for the classic Rosenbrock function. The above function is non-convex but has one minimum at $(1, 1)$. We also consider the combination of the above Rosenbrock function with the following 2D Sine function:
\[
s(x,y)=0.25(\sin(8(x-1)-\pi/2)+\sin(8(y-1)-\pi/2)+2).
\]
In this case $r(x,y)+s(x,y)$ still has a unique global minimum at $(1,1)$. But there are many local minima. We take 5K random training samples from the domain $[-2,-2]\times [-1, 3]$. The proposed BiLipNet contains two monotone and Lipschitz layers (i.e., $K=2$ for \eqref{eq:bi-lip}). The nonlinear block $\gH$ has 4 hidden layers of width 128. The model size is roughly 16K. The ICNN model has 8 hidden layers with width of 180. The MLP has hidden units of $[128, 256, 256, 512]$. We trained i-ResNet and i-DenseNet with different depth and width such that the total amount of parameters is comparable with BiLipNet. 

\paragraph{Parametric Rosenbrock function.} We consider the following parametric Rosenbrock function
\[
r(x,y;p)=\frac{1}{200}(x-a)^2+\frac{1}{2}(y-bx^2)^2,\quad p=(a,b)\in [-1,1]^2.
\]
We take 10K random training data. The partially BiLipNet contains 3 orthogonal layers, and 2 monotone and Lipschitz layers (the $\gH$ block of each layer has 4 hidden layer with width 128).  The bias term of each orthogonal layer is produced by an MLP with hidden units of $[64, 128, 2]$ while the bias for those hidden units inside the $\gH$ block is generated by an MLP of $[64, 128, 256, 512]$. The model's bi-Lipschitz bound is chosen to be $(0.04, 16)$. The resulting model size is 604K.

\paragraph{ND Rosenbrock function.} We also consider the $N$-dimensional (with $N=20$) Rosenbrock function:
\[
    R(x)=\frac{1}{N-1}\sum_{i=1}^{N-1}r(x_i, x_{i+1})
\]
which is non-convex and has a unique global minimum at $(1,1, \ldots, 1)$. Besides it also has many local minima. We take 10K random samples over the domain $[-2,2]^{20}$ and do training with batch size of 200. Note that the data size is very small compared to the dimension. We then use 500K samples for testing. BiLipNet has two monotone and Lipschitz layers (i.e., $K=2$ for \eqref{eq:bi-lip}) where each layer has a nonlinear block $\gH$ with 8 hidden layer of width 256 (model size $\sim$ 2.1M). For the i-ResNet/i-DenseNet, we try different depths from 2 to 10 and observe that depth of 5 yields slightly better results. The width of hidden layer is chosen so that it has a similar amount of parameters as BiLipNet.

\subsection{Extra results}\label{sec:extra-result}

Some extra results for the bi-Lipschitz models on two-moon and CIFAR-10/100 datasets are shown in \cref{fig:sngp-357} and \cref{tab:cifar10-100-bn}, respectively. \cref{fig:rosenbrock-v2} depicts the additional results on surrogate loss learning. 

\begin{table}[ht]
\centering
% \scalebox{0.7}{
\resizebox{\textwidth}{!}{  
\begin{tabular}{c|c|cc|cc|cc}
\toprule
& & \multicolumn{2}{c|}{Accuracy ($\uparrow$)} & 
\multicolumn{2}{c|}{ECE ($\downarrow$)} &
\multicolumn{2}{c}{NLL ($\downarrow$)} 
\\
Method & $c$  & Clean & Corrupted & Clean & Corrupted & Clean & Corrupted
\\
\midrule \midrule
\multicolumn{8}{c}{\textbf{CIFAR-10}} 
\\
\midrule
\multirow{3}{*}{SNGP-BN} & 0.95 & 94.7 $\pm$ 0.079 & 73.0 $\pm$ 0.461 & 0.017 $\pm$ 0.002 & 0.127 $\pm$ 0.010 & 0.166 $\pm$ 0.004 & 0.991 $\pm$ 0.054 \\
& 0.65 & 94.1 $\pm$ 0.159 & 72.3 $\pm$ 0.561 & 0.016 $\pm$ 0.000 & 0.116 $\pm$ 0.005 & 0.182 $\pm$ 0.005 & 0.985 $\pm$ 0.029 \\
& 0.35 & 92.3 $\pm$ 0.260 & 70.4 $\pm$ 0.800 & 0.008 $\pm$ 0.003 & 0.095 $\pm$ 0.007 & 0.231 $\pm$ 0.006 & 0.995 $\pm$ 0.031 \\
\midrule
\multirow{3}{*}{BiLipNet} & 0.95 & 86.2 $\pm$ 0.250 & 70.8 $\pm$ 0.469 & 0.020 $\pm$ 0.003 & 0.052 $\pm$ 0.005 & 0.423 $\pm$ 0.006 & 0.895 $\pm$ 0.020 \\
& 0.65 & 86.7 $\pm$ 0.129 & 72.8 $\pm$ 0.592 & 0.015 $\pm$ 0.005 & 0.047 $\pm$ 0.009 & 0.400 $\pm$ 0.006 & 0.830 $\pm$ 0.024 \\
& 0.35 & 84.5 $\pm$ 0.184 & 72.6 $\pm$ 0.216 & 0.010 $\pm$ 0.002 & 0.052 $\pm$ 0.004 & 0.457 $\pm$ 0.002 & 0.827 $\pm$ 0.008 \\
\midrule \midrule
\multicolumn{8}{c}{\textbf{CIFAR-100}} 
\\
\midrule
% cifar100, sn=0.95
\multirow{3}{*}{SNGP-BN} & 0.95 & 72.3 $\pm$ 0.513 & 44.8 $\pm$ 0.470 & 0.071 $\pm$ 0.006 & 0.091 $\pm$ 0.006 & 1.042 $\pm$ 0.018 & 2.476 $\pm$ 0.025 \\
& 0.65 & 67.8 $\pm$ 1.006 & 41.5 $\pm$ 0.916 & 0.117 $\pm$ 0.007 & 0.092 $\pm$ 0.002 & 1.231 $\pm$ 0.035 & 2.573 $\pm$ 0.036 \\
& 0.35 & 61.9 $\pm$ 0.741 & 37.0 $\pm$ 0.660 & 0.158 $\pm$ 0.006 & 0.098 $\pm$ 0.006 & 1.510 $\pm$ 0.029 & 2.760 $\pm$ 0.043 \\
\midrule
\multirow{3}{*}{BiLipNet} & 0.95 & 51.0 $\pm$ 0.480 & 35.8 $\pm$ 0.397 & 0.230 $\pm$ 0.006 & 0.137 $\pm$ 0.007 & 2.064 $\pm$ 0.024 & 2.718 $\pm$ 0.014 \\
& 0.65 & 55.2 $\pm$ 0.426 & 39.2 $\pm$ 0.495 & 0.225 $\pm$ 0.004 & 0.137 $\pm$ 0.005 & 1.887 $\pm$ 0.021 & 2.576 $\pm$ 0.022 \\
& 0.35 & 54.4 $\pm$ 0.438 & 41.1 $\pm$ 0.200 & 0.194 $\pm$ 0.008 & 0.126 $\pm$ 0.009 & 1.876 $\pm$ 0.031 & 2.447 $\pm$ 0.016 \\
\bottomrule
\end{tabular}
}
\caption{ 
Results for SNGP-BN (SNGP with batch normalization) and BiLipNet (\textit{without} batch normalization) on CIFAR-10/100, averaged over 5 seeds. As pointed out by \cite{liu2023simple}, the batch normalization may rescale a layer’s spectral norm in unexpected ways. So there is no theoretical guarantee on bi-Lipschitz property for SNGP-BN. This may offer it extra expressive power, leading to performance improvement in both clean and corrupted accuracy for a large distortion models (i.e. $c=0.95$). For models with low distortion (i.e. $c=0.35$), BiLipNet has better accuracy for the corrupted dataset.  
}
\label{tab:cifar10-100-bn}
\end{table}

\begin{figure}[!tb]
    \centering
    \includegraphics[width=0.95\linewidth]{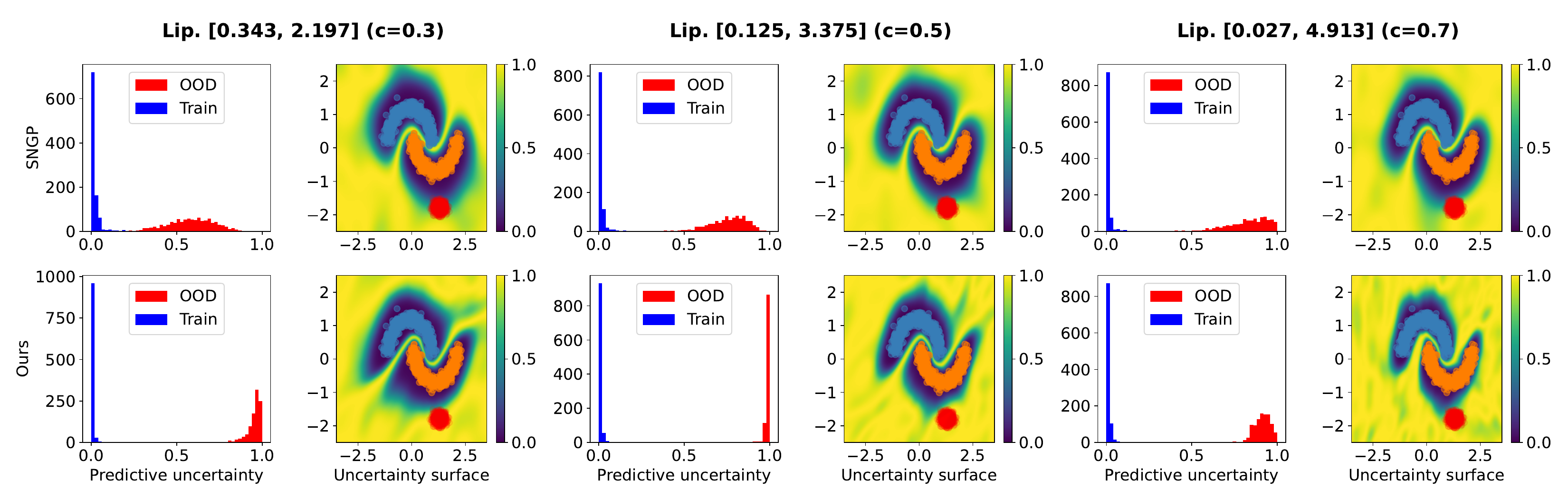}
    \caption{Uncertainty qualification via neural Gaussian process with different bi-Lipschitz bound specifications. }
    \label{fig:sngp-357}
\end{figure}

\begin{figure}[!tb]
    \centering
    \includegraphics[width=0.9\linewidth]{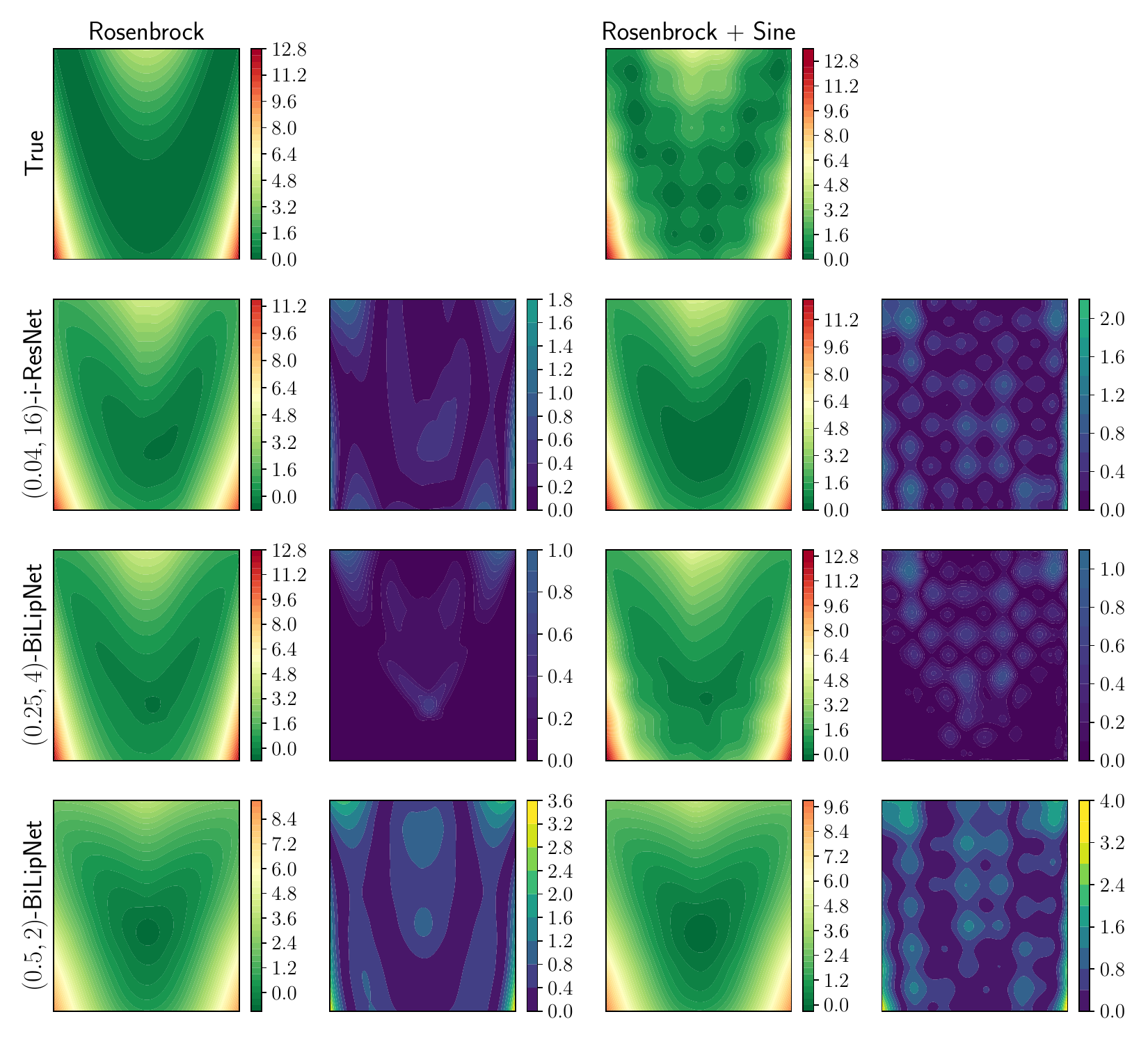} 
    \caption{Additional results for Learning a surrogate loss for the Rosenbrock and Rosenbrock + Sine functions. The first row contains the true functions while the remaining rows show learned functions and errors for various surrogate loss models. Our model (BiLipNet) has the flexibility of capturing the non-convex sub-level sets, but can also fit smoothed representations by reducing the distortion parameter.}
    \label{fig:rosenbrock-v2}
\end{figure}

%%%%%%%%%%%%%%%%%%%%%%%%%%%%%%%%%%%%%%%%%%%%%%%%%%%%%%%%%%%%%%%%%%%%%%%%%%%%%%%
%%%%%%%%%%%%%%%%%%%%%%%%%%%%%%%%%%%%%%%%%%%%%%%%%%%%%%%%%%%%%%%%%%%%%%%%%%%%%%%

\end{document}